\documentclass[preprint,twoside,11pt]{article}

\usepackage{jmlr2e_number}
\usepackage{blindtext}
\usepackage{additional_package}
\usepackage{amsmath}
\usepackage{xcolor} %%%
\usepackage{mathrsfs} % curly F for Fourier transform symbol
\usepackage{quiver} % commutative diagrams
\usepackage{enumitem}
\usepackage{hyperref}
\usepackage{thm-restate} % restate theorem

\usepackage{lastpage}
\jmlrheading{23}{2022}{1-\pageref{LastPage}}{1/21; Revised 5/22}{9/22}{21-0000}{Author One and Author Two}

\ShortHeadings{Geometric Learning with Positively Decomposable Kernels}{Da Costa, Mostajeran, Ortega, and Said}
\firstpageno{1}

\begin{document}

\title{Geometric Learning with Positively Decomposable Kernels}

\author{\name Natha\"el Da Costa \email nathael.dacosta@gmail.com \\
        \name Cyrus Mostajeran \email cyrus.mostajeran@gmail.com \\
        \name Juan-Pablo Ortega \email juan-pablo.ortega@ntu.edu.sg \\
       \addr Division of Mathematical Sciences\\
		School of Physical and Mathematical Sciences\\
       Nanyang Technological University\\
       21 Nanyang Link, 637371, Singapore
       \AND
        \name Salem Said \email salem.said@univ-grenoble-alpes.fr \\
       \addr Laboratoire Jean Kuntzman\\
       Universit\'e Grenoble-Alpes\\
       Grenoble, 38400, France
       }

%\editor{Jean-Philippe Vert}

\maketitle

\begin{abstract}%   <- trailing '%' for backward compatibility of .sty file
Kernel methods are powerful tools in machine learning. Classical kernel methods are based on positive definite kernels, which enable learning in reproducing kernel Hilbert spaces (RKHS). For non-Euclidean data spaces, positive definite kernels are difficult to come by. In this case, we propose the use of reproducing kernel Krein space (RKKS) based methods, which require only kernels that admit a positive decomposition. We show that one does not need to access this decomposition to learn in RKKS. We then investigate the conditions under which a kernel is positively decomposable. We show that invariant kernels admit a positive decomposition on homogeneous spaces under tractable regularity assumptions. This makes them much easier to construct than positive definite kernels, providing a route for learning with kernels for non-Euclidean data. By the same token, this provides theoretical foundations for RKKS-based methods in general.
\end{abstract}


{\tableofcontents}

\section{Introduction}
Kernel methods have proved to be successful in machine learning. Classically, kernel methods rely on a positive definite (PD) kernel on the data space. Such a PD kernel gives an embedding of the data space into a reproducing kernel Hilbert space (RKHS), which is a space of functions over the data space. Learning problems are then often phrased as learning a function in the RKHS, which can be done effectively thanks to the Representer Theorem {\citep{scholkopf_generalized_2001}}. {This theorem makes kernels a particularly attractive paradigm as, combined with universality results {\citep{steinwart2001influence, micchelli_universal_2006, sriperumbudur_universality_2011, simon-gabriel_kernel_2018}}, it provides a learning guarantee. In addition, kernels allow for uncertainty quantification by using Gaussian processes \citep{rasmussen_gaussian_2005, kanagawa_gaussian_2018}.}\par
Kernel methods on Euclidean data spaces have been thoroughly studied, and numerous families of PD kernels have been proposed in Euclidean spaces, many of which have universal approximation properties {\citep{steinwart2001influence, micchelli_universal_2006}}. However, in many applications, it is important to capture the geometry of the data and view it as lying on a non-Euclidean space, such as a manifold. There has been considerable recent interest in constructing PD kernels on manifolds {\citep{jayasumana_kernel_2013, feragen_geodesic_2014, borovitskiy_matern_2022, azangulov_stationary_2023, azangulov_stationary_2023-1, da_costa_invariant_2023}}. It has proved difficult to find closed-form PD kernels that are {easy to implement and evaluate} on general geometries. \par
\cite{ong_learning_2004} proposed a different class of kernel algorithms that do not require positive definiteness of the kernel. Instead, they require the existence of a positive (PD) decomposition for the kernel, meaning that it can be written as a difference of PD kernels. Such a kernel embeds the data space into a reproducing kernel Krein space (RKKS), consisting of functions on the data space. As in the RKHS case, this allows for the solution of learning problems from data {\citep{schleif_indefinite_2015, bonnet-loosli_learning_2016, oglic_learning_2018, liu_fast_2021,liu_analysis_2021}} via, for instance, an adapted representer theorem. \par
We are interested in the problem of characterizing kernels that admit a PD decomposition, both on Euclidean and non-Euclidean data spaces. Providing such a characterization promises to solve several open problems in the literature. More specifically, it achieves the following.\par
\begin{enumerate}
    \item \emph{It justifies the use of non-PD kernels for RKHS-based methods.} In applied areas, non-PD kernels have sometimes been used when restricted to a data set for which the Gram matrix corresponding to this kernel is PD. This has, for example, been done with the geodesic Gaussian kernel on non-Euclidean geometries {\citep{calinon_gaussians_2020,jousse_geodesic_2021}}, which is known not to be PD in general {\citep{feragen_geodesic_2014, da_costa_gaussian_2023, da_costa_invariant_2023, li_gaussian_2023}}. It seems that many RKHS-based methods (such as support vector machines {\citep{cristianini_support_2008}}) may be generalized to RKKS-based methods {\citep{bonnet-loosli_learning_2016}}, such that the solution to both learning problems are the same when restricted to a finite data set on which the kernel is PD. Therefore applying an RKHS-based method to a non-PD kernel, which is PD on a given finite data set, can be justified by the existence of a PD decomposition for the kernel.
    \item \emph{It justifies the use of RKKS-based methods.} Analogous to how a kernel may be PD on a finite data set but not on the whole space, a kernel may be PD decomposable on a finite data set but not on the whole space. Except that in this case, this subtlety is even easier to ignore, as kernels are \emph{always} PD decomposable on finite data sets. However, if a kernel does not admit a PD decomposition \emph{on the whole} data space, then the solution to the learning problem has no guarantees to generalize to unseen data. RKKS-based methods have so far largely ignored the issue of checking PD decomposability of the kernels used. Only linear combinations of PD kernels were truly known to admit a PD decomposition. Recently, a theoretical advance was made by \cite{liu_fast_2021} by observing that invariant kernels on $X=\R^n$ admit a PD decomposition if and only if they are the inverse Fourier transform of a finite signed measure. We will generalize this result to many non-Euclidean data spaces $X$.
    \item \emph{It motivates the use of RKKS-based methods for non-Euclidean data spaces.} In the Euclidean case $X=\R^n$, many families of kernels are known to be PD. RKHS-based methods are, therefore, usually sufficient in this case. In contrast, for non-Euclidean data spaces $X$, PD kernels can be challenging to construct. For example, it has been shown that the generalization of some of the most widely used kernels in $\R^n$, such as the Gaussian kernel, are often not PD on non-Euclidean geometries {\citep{feragen_geodesic_2014, da_costa_gaussian_2023, da_costa_invariant_2023, li_gaussian_2023}}. The positive definiteness of a kernel is a stringent condition and heavily depends on the specific geometry of the data space in question. This makes it difficult to describe general families of kernels that are PD on general classes of geometries. {Such families can sometimes be expressed as infinite series expansions on compact manifolds or as integrals on non-compact ones \citep{borovitskiy_matern_2022, azangulov_stationary_2023, azangulov_stationary_2023-1}}. Closed-form PD kernels, however, have only been found on a handful of geometries {\citep{feragen_geodesic_2014, da_costa_invariant_2023}}.
In contrast, we will see that PD decomposability is a much weaker condition to impose on a kernel, and it will be enough to require some symmetry and regularity assumptions. Therefore, in non-Euclidean settings, RKKS-based methods become particularly compelling.
\end{enumerate}

\paragraph{Paper organization and contributions:} This paper aims at providing verifiable sufficient conditions for kernels to admit a positive decomposition, thereby justifying their use for RKKS learning. We pay particular attention to \emph{invariant} kernels on locally compact groups $G$ and their coset spaces $G/H$, and then look into the special case where $G$ is a Lie group, and $G/H$ is a homogeneous space. \par
We start by briefly reviewing the theory of PD kernels and how they enable RKHS learning in Section \ref{pd_ker_sec}. We then describe the analogous theory of PD decomposable kernels and RKKS learning in Section \ref{pdd_ker_sec}. In particular, {we formalize} a general representer theorem for RKKS learning. Crucially, we show that, when the regularizer of the learning problem is linear in the squared indefinite inner product, one does not need access to the PD decomposition of the kernel in order to apply the representer theorem; one only needs to know that it exists. {We then discuss the extent to which RKHS methods can be adapted to the RKKS framework. In Section \ref{pd_decomposition_sec} we review the few examples where PD decompositions are known to exist.} \par
In Section \ref{sta_ker_sec} we focus on invariant kernels on homogeneous spaces of locally compact groups and prove the correspondence between (PD/PD decomposable) invariant kernels on $G/H$ and (PD/PD decomposable) Hermitian functions on the double coset space $H\backslash G/H$. When $H=\{e\}$, this allows us to relate PD decomposable invariant kernels on $G$ to real-valued functions in the Fourier-Stieltjes algebra $B_\C(G)$, which in turn allows us to leverage results from harmonic analysis. \par
In Section \ref{com_sec}, we study the case where $G$ is commutative, where the situation simplifies considerably. In particular, we review known sufficient conditions for functions to belong to the spaces $B_\C(\R^n)$ and $B_\C(S^1)$. As an example, we show that the Gaussian kernel, despite failing to be PD, for instance, on the torus $\T^n$, has a PD decomposition on any Abelian Lie group. \par
In Section \ref{non_com_sec}, we move to the case where $G$ is not assumed to be commutative. In this case the literature has focused on the space of functions $B_\C(G)$, and not the more general $B_\C(H\backslash G/H)$. To fill this gap, we show that any function that is invariant on (double) cosets and is PD decomposable has a PD decomposition into functions that are invariant on (double) cosets, i.e. $B_\C(G)\cap C{_\C}(G/H) = B_\C(G/H)$ and $B_\C(G)\cap C{_\C}(H\backslash G/H) = B_\C(H\backslash G/H)$. This allows us to show in Section \ref{suf_cond_sec} that on a homogeneous space $G/H$, with $G$ a unimodular Lie group and $H$ a compact Lie subgroup, smooth functions whose derivatives decay appropriately at infinity admit PD decompositions. As an example, we then show that the Gaussian kernel has a PD decomposition on non-compact symmetric spaces.
\section{Kernels and Learning}\label{ker_sec}
\subsection{{Learning with} positive definite kernels}\label{pd_ker_sec}
\begin{definition}\label{kernel_def}
A kernel on a set $X$ is a Hermitian map $k:X\times X \rightarrow \C$.
\end{definition}
In practice, we are interested in real-valued kernels $k:X\times X \to \R$, but allowing kernels to take complex values will be convenient in the following theoretical work.
{
\begin{definition}
A kernel $k$ on a set $X$ is said to be positive definite (PD) if for all $N\in\N$, $x_1,\dots,x_N\in X$ and all $c_1,\dots,c_N\in\C$,
$$\sum_{i=1}^{N}\sum_{j=1}^{N}\overline {c_i}c_jk(x_i,x_j)\geq 0$$
i.e. the matrix $\big(k(x_i,x_j)\big)_{i,j}$, which we call the Gram matrix of $x_1,\dots,x_N$, is Hermitian positive semidefinite.
\end{definition}}
We refer to {\citet[Chapter 3]{berg_harmonic_1984}} for the general theory of PD kernels.

In machine learning, a key reason for the importance of PD kernels is the following theorem (see {\citet[Theorem 3.16]{paulsen_introduction_2016}}) that guarantees that PD kernels have a natural associated reproducing kernel Hilbert space (RKHS):
\begin{theorem}\label{rkhs_thm}
Let $k$ be a PD kernel on a set $X$. Then there is a complex Hilbert space $\H$ {of complex-valued functions}, which we call the reproducing kernel Hilbert space (RKHS) associated with $k$, with a map
\begin{equation}
\label{canonical feature map}
\begin{aligned}
\Phi: X&\to \H \\
x &\mapsto k(x,\cdot )
\end{aligned}
\end{equation}
such that
\begin{enumerate}
\item $\langle \Phi(x),\Phi(y) \rangle_\H = k(x,y)$ for all $x,y\in X$,
\item $\overline{\spn(\Phi(X))} = \H$.
\end{enumerate}
\end{theorem}
In the above theorem $\overline{\spn(\Phi(X))}$ stands for the completion of $\spn(\Phi(X))$ with respect to the RKHS norm.

Any map $\Phi $ that satisfies the equality $\langle \Phi(x),\Phi(y) \rangle_\H = k(x,y)$ in the first part of the theorem is called a {\it feature map } of the RKHS ${\mathcal H} $. The functions in the image of the feature map $\Phi  $ in \eqref{canonical feature map} are called {\it kernel sections}{, or {\it canonical basis functions},} of $k$. 
In practice, kernels will be real-valued, so we will usually consider the induced \emph{real} RKHS{, defined as the subspace of real-valued functions of $\mathcal H$}. \par
Suppose we are trying to learn a real or complex-valued function $f$ over $X$ from a finite set of observations. One can try to find a good approximation for this function in the RKHS associated with a particular kernel map $k$. This is particularly pertinent when the kernel in question has universality properties {\citep{steinwart2001influence, micchelli_universal_2006, sriperumbudur_universality_2011, simon-gabriel_kernel_2018}}, that is, when its sections are dense, say, in the set of continuous functions on compact subsets of $X$. The representer theorem in RKHS {\citep{scholkopf_generalized_2001}} tells us that, for a general class of learning problems, the best approximation $f^*$ in the RKHS, given finitely many observations of the function, is a finite linear combination of the kernel sections $k(x_i,\cdot )$. {More specifically, given
\begin{enumerate}
    \item $\mathcal D = \{(x_1,y_1),\dots, (x_N,y_N)\}\subset (X\times\R)^N$ a finite data set,
    \item $g:[0,\infty)\to \R$ a strictly increasing function,
    \item $L:\H \times(X\times\R)^N \to \R$ a loss functional, determined exclusively through function evaluations, i.e. if $f,g\in\H$ are such that $f|_{\mathcal D} = g|_{\mathcal D}$, then $L(f,\mathcal D) = L(g,\mathcal D)$,
    \item $\Omega(\mathcal D) \subset \H$ a feasible set, determined exclusively through function evaluations, i.e. if $f,g\in\H$ are such that $f|_{\mathcal D} = g|_{\mathcal D}$, then $f\in\Omega(\mathcal D)$ if and only if $g\in\Omega(\mathcal D)$.
\end{enumerate}
Then any solution to the minimization problem
\begin{equation}
\begin{aligned}
\label{optimisation problem RKHS}
    \underset{f\in\H}{\text{\rm{minimize}}}&\; L(f,\mathcal D) +g(\|f\|_\H) \\
    \text{s.t.}&\; f\in \Omega(\mathcal D)
\end{aligned}
\end{equation}
has the form
$$f^\ast  = \sum_{i=1}^N\alpha_i k(x_i,\cdot )$$
for some $\alpha_1, \ldots, \alpha _N\in\R$.}

The problem is therefore reduced to learning the coefficients $\alpha_i$. We have thus transformed the infinite-dimensional problem of learning a function into a finite-dimensional linear algebra problem.
{
\begin{example} \normalfont\label{krr_ex}
    Kernel ridge regression (KRR) is a learning algorithm for fitting a function in the RKHS to noisy observations. Given a PD kernel $k$ on a set $X$, $\H$ the RKHS associated to $k$, and a finite data set $\mathcal D = \{(x_1,y_1),\dots, (x_N,y_N)\} \subset (X\times \R)^N$, the learning problem solved by KRR is given by
    \begin{equation*}
        \underset{f\in\H}{\text{\rm{minimize}}} \frac{1}{N}\sum_{i=1}^N(f(x_i) - y_i)^2 + c\|f\|_\H^2
    \end{equation*}
    for some $c > 0$. The representer theorem gives us a closed-form solution to this problem as $f^\ast  = \sum_{i=1}^N\alpha_i k(x_i,\cdot )$ for some $\alpha_i \in \R$, namely $\alpha = (K+NcI)^{-1}y$, where $K = \big(k(x_i,x_j)\big)_{i,j}$ is the Gram matrix and $y=(y_i)_i$.
\end{example}}
\begin{example} \normalfont\label{svm_ex}
 Support vector machines (SVM) {\citep{cristianini_support_2008}} are a class of learning algorithms for classification. Given a finite data set $\mathcal D = \{(x_1,y_1),\dots, (x_N,y_N)\} \subset (X\times \{-1,1\})^N$, and $k$, $X$, $\mathcal H$ as above, the learning problem solved by binary SVM can be written as
 \begin{equation}\label{svm_eq}
     \underset{f\in\H, \;b\in \R}{\text{\rm{minimize}}}\; \|f\|_\H^2 \quad \mbox{s.t.} \quad\sum_{i=1}^N \max(0, 1-y_i(f(x_i)+b))  \leq \tau
 \end{equation}
 for some fixed parameter $\tau >0$. Once again, the representer theorem gives us a closed-form solution to this problem. Note that to be precise, the parametric extension to the representer theorem {\citep[Theorem 2]{scholkopf_generalized_2001}} is applied here, since the minimization is over $f$ and $b$, and not solely over $f$.
\end{example}

{
The field of positive definite kernel methods extends beyond the use of the representer theorem. Gaussian processes (GPs) are the Bayesian analogue to frequentist kernel methods, with the posterior mean of GP regression corresponding to the solution KRR being one of many correspondences between the two paradigms \citep{kanagawa_gaussian_2018}. Kernels can also be used to define distances on the space of probability measures through maximum mean discrepancy \citep{gretton_kernel_2012}.}
\subsection{{Learning with} positively decomposable kernels}\label{pdd_ker_sec}
A natural question that can be posed is whether the assumption on the kernel $k$ being positive definite is necessary. That is, can we still obtain a representer theorem while dropping this assumption? It turns out that we can. For this, we need to define the notion of a Krein space. We refer to \cite{bognar_indefinite_1974} for a thorough presentation of Krein spaces and to \cite{schwartz_sous-espaces_1964} and \cite{alpay_remarks_1991} for the study of their reproducing kernels. This last reference contains a good summary of Schwartz's contributions.
\begin{definition}\label{krein_def}
An indefinite inner product on a real vector space $\K$ is a Hermitian bilinear map $\langle\cdot ,\cdot \rangle: X\times X \to \R$ which is non-degenerate, that is, for all $f\in \K$,
$$\langle f,f\rangle = 0 \quad \mbox{implies that} \quad f =0.$$
A complex vector space $\K$ equipped with an indefinite inner product $\langle\cdot ,\cdot \rangle$ is called a Krein space if it can be written as the algebraic direct sum
\begin{equation}\label{decomposition_eq}
    \K = \H_+\oplus \H_-
\end{equation}
such that
\begin{enumerate}
\item $\H_+$ equipped with $\langle\cdot ,\cdot \rangle_+:=\langle\cdot ,\cdot \rangle$ is a Hilbert space,
\item $\H_-$ equipped with $\langle\cdot ,\cdot \rangle_-:=-\langle\cdot ,\cdot \rangle$ is a Hilbert space,
\item $\langle f_+,f_-\rangle=0$ for all $f_+\in\H_+$, $f_-\in\H_-$. 
\end{enumerate}
In other words, there are complete inner products $\langle\cdot ,\cdot \rangle_+$, $\langle\cdot ,\cdot \rangle_-$ on $\H_+$, $\H_-$ respectively such that
$$\langle f,g\rangle = \langle f_+,g_+\rangle_+ - \langle f_-,g_-\rangle_-$$
for all $f = f_++f_-$, $g= g_++g_-$ with $f,g\in\K$, $f_+,g_+\in\H_+$, $f_-,g_-\in\H_-$. \par
$\K$ is equipped with the product topology on $\H_+ \times \H_-$, which can be shown to be independent of the choice of decomposition (\ref{decomposition_eq}).
\end{definition}
\begin{definition}\label{decomposition_def}
Let $k$ be a kernel on $X$. Then $k$ is said to have a positive (PD) decomposition if it can be written as
$$k = k_+ - k_-$$
where $k_+$ and $k_-$ are PD kernels.
\end{definition}
\begin{remark} \normalfont PD decompositions are also sometimes called {\it fundamental decompositions} {\citep{bognar_indefinite_1974}}, or {\it Kolmogorov decompositions} {\citep{mary_hilbertian_2003}}.
\end{remark}
With this notion, we obtain an analogous result to Theorem \ref{rkhs_thm} (see {\citet[Theorem 2.1]{alpay_remarks_1991}}):

\begin{theorem}
\label{rkks_thm}
Let $k$ be a PD decomposable kernel on a set $X$. Then there is a complex Krein space $\K$ {of complex-valued functions}, which we call the reproducing kernel Krein space (RKKS) associated to $k$, with a map
\begin{equation*}
\begin{aligned}
\Phi: X&\to \K \\
x &\mapsto k(x,\cdot ).
\end{aligned}
\end{equation*}
such that
\begin{enumerate}
\item $\langle \Phi(x),\Phi(y) \rangle_\K = k(x,y)$ for all $x,y\in X$,
\item $\overline{\spn(\Phi(X))} = \K$.
\end{enumerate}
\end{theorem}

\begin{remark} \normalfont\label{non_unique_rmk}
As opposed to the PD case {where the RKHS associated to a kernel is unique (see \citet[Theorem 3.16 \& following Remark]{paulsen_introduction_2016})}, the RKKS associated to a PD decomposable kernel $k$ can, in some {(arguably atypical)} cases, be non-unique; see \cite{schwartz_sous-espaces_1964}  and \cite{alpay_remarks_1991}.
\end{remark}
Theorem \ref{rkks_thm} can be made into an if and only if statement: $k$ must admit a PD decomposition in order to give rise to an RKKS, in the same way that $k$ must be PD in order to give rise to an RKHS. {Indeed, a PD decomposition for $k$ can be obtained given a decomposition of the RKKS as in (\ref{decomposition_eq}).}

As in the PD case, we will usually consider only the induced \emph{real} RKKS of a real-valued PD decomposable kernel{, defined as the subspace of real-valued functions of $\mathcal K$}. \par
Now, as in the RKHS case, a representer theorem can be formulated in the context of RKKS. {Such a representer theorem in RKKS was formulated in \citet[Theorem 11]{ong_learning_2004}. The indefinite nature of the inner product implies that the optimization problem \eqref{optimisation problem RKHS} is replaced by a critical point problem. More specifically, given
\begin{enumerate}
    \item $\mathcal D = \{(x_1,y_1),\dots, (x_N,y_N)\}\subset (X\times\R)^N$ a finite data set,
    \item $g:\R\to \R$ a strictly monotonic differentiable function,
    \item $L:\K \times (X\times\R)^N \to \R$ a loss functional, determined exclusively through function evaluations, Fréchet differentiable in the first argument,
    \item $I:\K \times (X\times\R)^N \to \R^m$ and $E:\K \times (X\times\R)^N \to \R^l$ functionals determined exclusively through function evaluations and Fréchet differentiable in the first argument,
\end{enumerate}
under some appropriate additional regularity condition on $I$ and $E$ (see Appendix \hyperref[representer_app]{A}), any solution to the stabilization (critical point) problem
\begin{equation}\label{rkks_opti_eq}
\begin{aligned}
    \underset{f\in\K}{\text{stabilize}}&\; L(f,\mathcal D) +g(\langle f,f\rangle_\K) \\
    \text{s.t.}&\; I(f,\mathcal D) \leq 0 \\
    &\; E(f,\mathcal D) = 0
\end{aligned}
\end{equation}
has the form
\begin{equation*}
f ^\ast  = \sum_{i=1}^N\alpha_i k(x_i,\cdot ),
\end{equation*}
for some $\alpha_1,\dots,\alpha_N\in \R$. Specifically, $f^*$ is given in closed form by
\begin{equation*}
f ^* =- \frac{1}{2 g'(\left\langle f^\ast,f^\ast\right\rangle_{{\mathcal K}})}\sum _{i=1} ^N \left(\partial _i L(f^*, \mathcal D) + \lambda ^{\top}\partial _i I(f^*,\mathcal D)+\mu ^{\top}\partial _i E(f^*,\mathcal D)\right)  k(x _i, \cdot )
\end{equation*}
for some Lagrange multipliers $\lambda\in \R^m$, $\mu\in \R^l$, where the $\partial_i$ denotes the derivative with respect to the $i$\textsuperscript{th} evaluation. Importantly, note that in the generic case $g(x) = cx$ for some $c\in \R$ and all $x\in \R$, we have $g'(\left\langle f^\ast,f^\ast\right\rangle_{{\mathcal K}}) = c$, and hence the $\alpha_i$ do not depend on the choice of indefinite inner product $\langle \cdot ,\cdot \rangle_\K$ associated to $\K$. This means that in this case \emph{one does not need to access the positive decomposition of the kernel in order to compute the solution, one only needs to know of its existence}.
}

We emphasize that, in view of this result, the situation for RKKS is analogous to the RKHS case in the sense that the solution of the stabilization problem is {\it in the span of the data}, although this adapted representer theorem is designed for the study of critical points, as opposed to minima of a loss functional, and thus algorithms must be adapted to this setting {\citep{hassibi_indefinite-quadratic_1999}}.

The reason for stabilizing, as opposed to minimizing, the regularized loss functional $L(f,\mathcal D) +g(\langle f,f\rangle_\K)$ is not that the proof of the RKKS representer theorem fails for minimization; it doesn't. Rather, the reason is that it is not expected that the regularized loss functional admits a minimum since $\langle f,f\rangle_\K$, as opposed to $\|f\|_\H$, is not bounded below for $f\in\K$. Thus, solutions to learning problems in RKKS may take the form of a saddle point instead.

{
Note that we have not described precisely what is meant by a critical point of the learning problem \eqref{rkks_opti_eq}. This is a subject that has not received much attention in the RKKS learning literature, possibly due to the difficulty of defining a critical point of a constrained optimization problem in function space. In Appendix \hyperref[representer_app]{A}, we thus expand on the work from \citet[Theorem 11]{ong_learning_2004} by providing a formal definition of such a critical point and outlining a proof for the resulting representer theorem.
}

{
Below we present a few examples of the use of the RKKS representer theorem.
\begin{example} \normalfont\label{kkrr_ex}
    Kernel ridge regression (c.f. Example \ref{krr_ex}) can be generalized to PD decomposable kernels \citep{ong_learning_2004}. Given a PD decomposable kernel $k$ on a set $X$, $\K$ the RKKS associated to $k$, and a finite data set $\mathcal D = \{(x_1,y_1),\dots, (x_N,y_N)\} \subset (X\times \R)^N$, the learning problem solved by KRR is given by
    \begin{equation*}
        \underset{f\in\H}{\text{\rm{stabilize}}} \frac{1}{N}\sum_{i=1}^N(f(x_i) - y_i)^2 + c\langle f,f\rangle_\K
    \end{equation*}
    for some $c \in \R$. The representer theorem gives us a closed-form solution to this problem as $f^\ast  = \sum_{i=1}^N\alpha_i k(x_i,\cdot )$ for some $\alpha_i\in\R$, namely $\alpha = (K+NcI)^{-1}y$, where $K = \big(k(x_i,x_j)\big)_{i,j}$ is the Gram matrix and $y=(y_i)_i$, assuming $-Nc$ is not an eigenvalue of $K$.
\end{example}}

\begin{example} \normalfont\label{ksvm_ex}
Support vector machines (c.f. Example \ref{svm_ex}) have been generalized to PD decomposable kernels in \cite{bonnet-loosli_learning_2016}. Given a finite data set $\mathcal D = \{(x_1,y_1),\dots, (x_N,y_N)\}$ $\subset (X\times \{-1,1\})^N$, and $k$, $X$, $\mathcal K$ as above, the analogous version to the learning problem (\ref{svm_eq}) can be written as
 \begin{equation*}
     \underset{f\in\K, \;b\in \R}{\text{\rm{stabilize}}}\;\langle f,f\rangle_\K 
    \quad\text{\rm{s.t.}}\quad\sum_{i=1}^N \max(0, 1-y_i(f(x_i)+b))  \leq \tau,
\end{equation*}
 for some fixed parameter $\tau >0$. The representer theorem tells us that any solution $(f,b)$ to this problem has the form $f= \sum_{i=1}^N\alpha_ik(x_i,\cdot )$ for some $\alpha_1, \ldots, \alpha _N\in\R$. We emphasize again that we are applying here the semiparametric extension to the representer theorem \cite[Theorem 12]{ong_learning_2004}, since the stabilization is over $f$ and $b$, and not solely over $f$.

 See Figure \ref{svm_fig} for an application.
 \end{example}

\begin{example}\normalfont{
    More generally, given an RKHS $\H$ and an RKKS $\K$, and a learning problem
    \begin{equation*}
    \begin{aligned}
    \underset{f\in\H}{\text{minimize}}&\; L(f,\mathcal D) +g(\|f\|^2_\H) \\
    \text{s.t.}&\; I(f,\mathcal D) \leq 0 \\
    &\; E(f,\mathcal D) = 0,
    \end{aligned}
    \end{equation*}
    one can consider the learning problem in RKKS
    \begin{equation*}
    \begin{aligned}
    \underset{f\in\K}{\text{stabilize}}&\; L(f,\mathcal D) +g(\langle f,f\rangle_\K) \\
    \text{s.t.}&\; I(f,\mathcal D) \leq 0 \\
    &\; E(f,\mathcal D) = 0.
    \end{aligned}
    \end{equation*}
This gives a general recipe for turning an RKHS problem into an RKKS problem. While Examples \ref{kkrr_ex} and \ref{ksvm_ex} follow this procedure, it appears to be an open question whether, starting from an arbitrary RKHS problem, the solutions to the corresponding RKKS problem behave as one would hope from solutions to the RKHS problem (eg. provide a good regression/classification rule).}

{
Turning an RKHS problem into an RKKS problem achieves another goal: suppose that, as in \citet{calinon_gaussians_2020,jousse_geodesic_2021}, we use a non-PD kernel, which has a PD Gram matrix when restricted to a given dataset, and consider a representer solution, as if the kernel was PD on the whole space. Through PD kernel theory, this has no theoretical backing: the kernel has no underlying RKHS over the whole space, therefore one cannot write an RKHS minimization problem that the representer solution solves. However, assuming that the kernel has a PD decomposition and that the appropriate regularity conditions on $L$, $I$, $E$ are satisfied, the representer solution is actually a critical point of the corresponding RKKS problem!}
\end{example}
\begin{figure}
\centering
\includegraphics[width=1\linewidth]{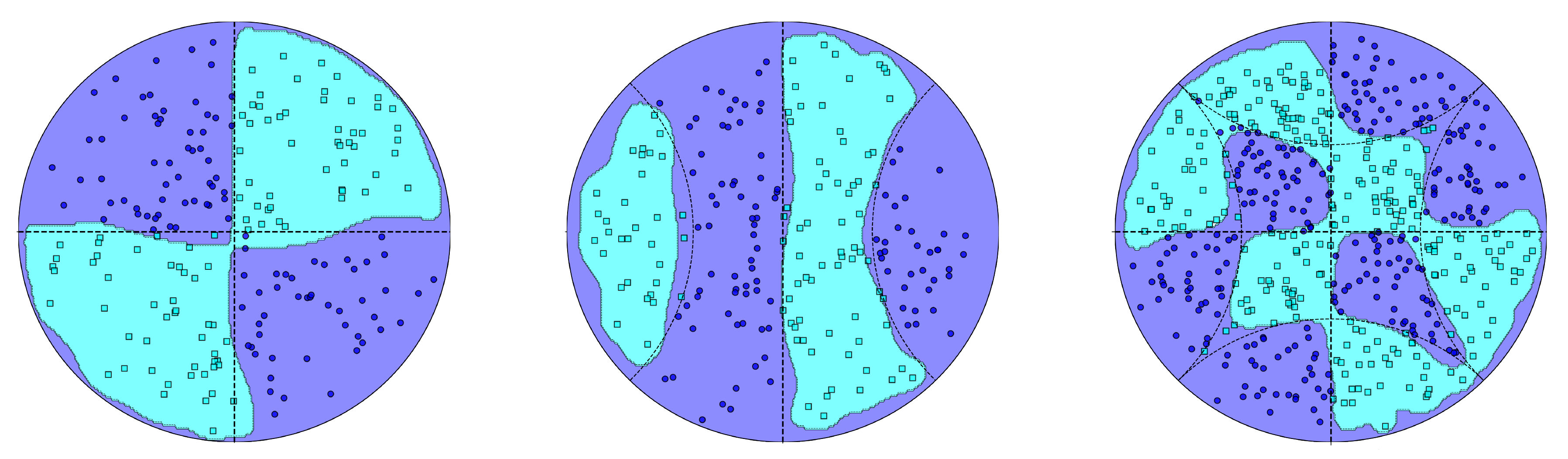}
\caption{The Krein SVM algorithm {\citep{bonnet-loosli_learning_2016}} applied on the hyperbolic plane $\Hyp^2$, with the geodesic Gaussian kernel $k = \exp(-\lambda d(\cdot ,\cdot )^2)$. The data is sampled from a Riemannian Gaussian distribution {\citep{Said2018,HOS2022,Said2023}} centered at the origin of the Poincare disc and is split into two classes according to geodesic decision boundaries (dotted curves in the figure). The number of sampled data points is 200, 200, and 500, respectively. The results of the classification are displayed in the Poincare disc model of $\Hyp^2$. We will show in Corollary \ref{symmetric_cor} that the Gaussian kernel admits a PD decomposition on $\Hyp^2$, justifying its use in this scenario.}
\label{svm_fig}
\end{figure}

{
The vast majority of optimization research has focused on minimization or maximization and rarely on stabilization. Stabilization problems are nonetheless common in fields such as Lagrangian and Hamiltonian mechanics where the least action principles \citep{Abraham1978, Marsden1994} that are at the core of the formulation of various physical theories adopt a stabilization rather than a minimization form. It is in contexts of this type, that in addition are rarely Euclidean, where the RKKS framework could be particularly appropriate}.

{
Nevertheless, designing minimization problems in RKKS could allow one to leverage existing optimization and RKHS learning literature more easily. This is possible, by replacing the regularizer $g(\langle\cdot,\cdot\rangle_\K) = g(\langle\cdot,\cdot\rangle_+ - \langle \cdot,\cdot\rangle_-)$ in \eqref{rkks_opti_eq} by a regularizer $g_+(\langle\cdot,\cdot\rangle_+)+g_-(\langle\cdot,\cdot\rangle_-)$ for some strictly increasing $g_+$, $g_-$. With respect to such a regularizer, one is then able to solve minimization problems. However, this requires a choice of inner product $\langle\cdot,\cdot\rangle_\K$ for the RKKS, i.e. a choice of PD decomposition for the kernel. Constructive choices of PD decompositions for kernels are beyond the scope of the present work.
\begin{example} \normalfont
    Krein kernel ridge regression introduced in \citet{oglic_learning_2018}, is an alternative generalization of kernel ridge regression to Krein spaces (c.f. Example \ref{kkrr_ex}). It refers to the learning problem
    \begin{equation*}
    \begin{aligned}
        \underset{f\in\K}{\text{\rm{minimize}}}&\; \sum_{i=1}^N(f(x_i) - y_i)^2 + c_+\langle f,f\rangle_+ +c_-\langle f,f\rangle_-\\
        \text{s.t.}&\;\sum_{i=1}^N\frac{1}{N}\left(f(x_i)-\sum_{j=1}^N\frac{1}{N}f(x_j)\right)^2 = r^2
        \end{aligned}
    \end{equation*}
    for some $r>0$.
\end{example}}

{
We conclude this section by noting that it is unclear at present whether other PD kernel non-representer-based approaches, such as Gaussian processes and maximum mean discrepancy, have analogues in the PD decomposable context. This appears to be an interesting avenue for future research.
}

{
\subsection{Positive decompositions}\label{pd_decomposition_sec}}
Without a PD decomposition, we do not have a representer theorem. In that case, it is not so much that the proof of the representer theorem fails; rather, we cannot state it in the first place. Indeed, the kernel does not have a reproducing property as in Theorem \ref{rkks_thm}, so there is no obvious concept of a space of functions we can stabilize over. Therefore, a solution found, for instance, in Example \ref{ksvm_ex}, may not come with a guarantee to provide a good classification rule for unseen data. 
This gives rise to the following fundamental question: \newline
\begin{problem}
\normalfont
\label{problem}
When does a kernel $k$ admit a PD decomposition?
\end{problem}
\begin{remark} 
\normalfont
Problem \ref{problem} was first studied by \cite{schwartz_sous-espaces_1964}. His setting was more general than ours, and only abstract conditions for PD decomposability were obtained. For instance, it is stated in that reference that $k$ has a PD decomposition if and only if there is a PD kernel $k_+$ such that $k_+-k$ is PD (see {\citet[Proposition 38 \& Proposition 23]{schwartz_sous-espaces_1964}}). However, if one wants to recognize PD decomposable kernels in practice, this condition does not appear to be any easier to work with than Definition \ref{decomposition_def} itself.
\end{remark}
\begin{example} \normalfont
\label{span_ex}
If $k_1,\dots,k_n$ are PD kernels and $a_1,\dots,a_n\in\R$, then the kernel
$$\sum_{i=1}^{n}a_ik_i$$
has the PD decomposition
$$\sum_{i: a_i \geq 0}a_ik_i - \sum_{i: a_i < 0}(-a_i)k_i.$$
In fact, observe that the space of kernels that admit a PD decomposition is exactly the real span of the PD kernels.
\end{example}
\begin{example} \normalfont\label{limit_ex}
    The pointwise limit of PD kernels is PD. It follows that the pointwise limit of PD decomposable kernels is PD decomposable.
\end{example}

\begin{remark} \normalfont
PD kernels form a convex cone that is closed under pointwise convergence. Now, Example \ref{span_ex} and Example \ref{limit_ex} show that PD decomposable kernels are the real linear span of these and thus form a real vector space closed under pointwise convergence.
\end{remark}
\begin{example} \normalfont\label{product_ex}
Similarly, since the pointwise product of PD kernels is PD \cite[Chapter 3 Theorem 1.12]{berg_harmonic_1984}, so it follows that the pointwise product of PD decomposable kernels is PD decomposable.
\end{example}
\begin{example} \normalfont\label{finite_ex}
Suppose $X$ is finite. Then a kernel $k$ can be viewed as a Hermitian matrix $K$. By Sylvester's law of inertia, we can write
$$K = A^\dag EA$$
for some invertible real matrix $A$, where $A^\dag$ denotes the conjugate transpose of $A$, and $E$ is of the form
\begin{equation*}
E =
\begin{bmatrix} 
    I_{N_+}&&   \\
    &-I_{N_-}&   \\
    &&0_M
\end{bmatrix}
\end{equation*}
for some $N_+,N_-,M\in \Z_{\geq 0}$, where $I_{N+}, I_{N-}$ denote the identity matrices with the corresponding sizes, and $0_M$ is the zero matrix of size $M\times M$. Then $E = E_+-E_-$ where
\begin{equation*}
E_+ =
\begin{bmatrix} 
    0_{N_+}&&   \\
    &-I_{N_-}&   \\
    &&0_M
\end{bmatrix}, \;
E_- =
\begin{bmatrix} 
    I_{N_+}&&   \\
    &0_{N_-}&   \\
    &&0_M
\end{bmatrix}.
\end{equation*}
This implies that
$$K = A^\dag E_+A-A^\dag E_-A,$$
and hence $k$ has a PD decomposition. However, when $X$ is infinite, one should not expect all kernels to admit a PD decomposition.
\end{example}

\begin{example} \normalfont
We present an example of a finite-dimensional RKKS associated with a non-finite $X$, namely $X=\Hyp^n$, the $n$-dimensional hyperbolic space. 
To motivate our construction, we recall that when $X = \R^n$, the standard inner product kernel $\langle \cdot, \cdot  \rangle: \R^n\times \R^n \to \R$ is PD and gives rise to the RKHS $\H = \{ \langle x, \cdot \rangle: x\in \R^n\}$, which consists of all linear maps $\R^n \to \R$. So for $f\in \H$, $b \in \mathbb{R} $, the set $\{x\in \R^n : f(x) = b\}$ is a hyperplane in $\R^n$. \par
Taking $X = \Hyp^n$, it is therefore natural to ask if there is a kernel that gives rise to an RKHS of functions $f$ such that for any $b \in \mathbb{R}  $, the sets $\{x\in\Hyp^n: f(x)=b\}$ are geodesic hyperplanes. This has practical relevance, as it would allow us to perform large-margin classification with geodesic decision boundaries, as in \cite{cho_large-margin_2019}. This is done naturally with a PD decomposable kernel giving rise to an RKKS. First consider the Minkowski inner product on $\R^{n+1}$:
$$(x,y) = x_0y_0-x_1y_1-\dots-x_ny_n$$
for all $x =(x_0,\dots,x_n), y= (y_0,\dots, y_n) \in \R^{n+1}$. Then, we can view $\Hyp^n$ as a hypersurface in this Minkowski space:
$$\Hyp^n = \{x\in \R^{n+1}: (x,x) = 1, x_0>0\}.$$
This is the hyperboloid model of $\Hyp^n$. The geodesic hyperplanes in $\Hyp^n$ are given by
$$\{x\in \Hyp^n: (y,x) = b\}$$
for some $y\in\R^{n+1}$, where the above is non-empty if and only if $(y,y) < 0$. These are the intersection of planes through the origin with the hyperboloid model of $\Hyp^n$ (see Figure \ref{hyperbolic_fig}). So the kernel $k: \Hyp^n\times\Hyp^n\to\R$ of interest can be taken to be $k = (\cdot ,\cdot )$. It has a PD decomposition by definition of the Minkowski inner product, and the associated real RKKS is the Minkowski space $\R^{n+1}$ equipped with $(\cdot ,\cdot )$.
\end{example}

\begin{figure}
\centering
\includegraphics[width=0.7\linewidth]{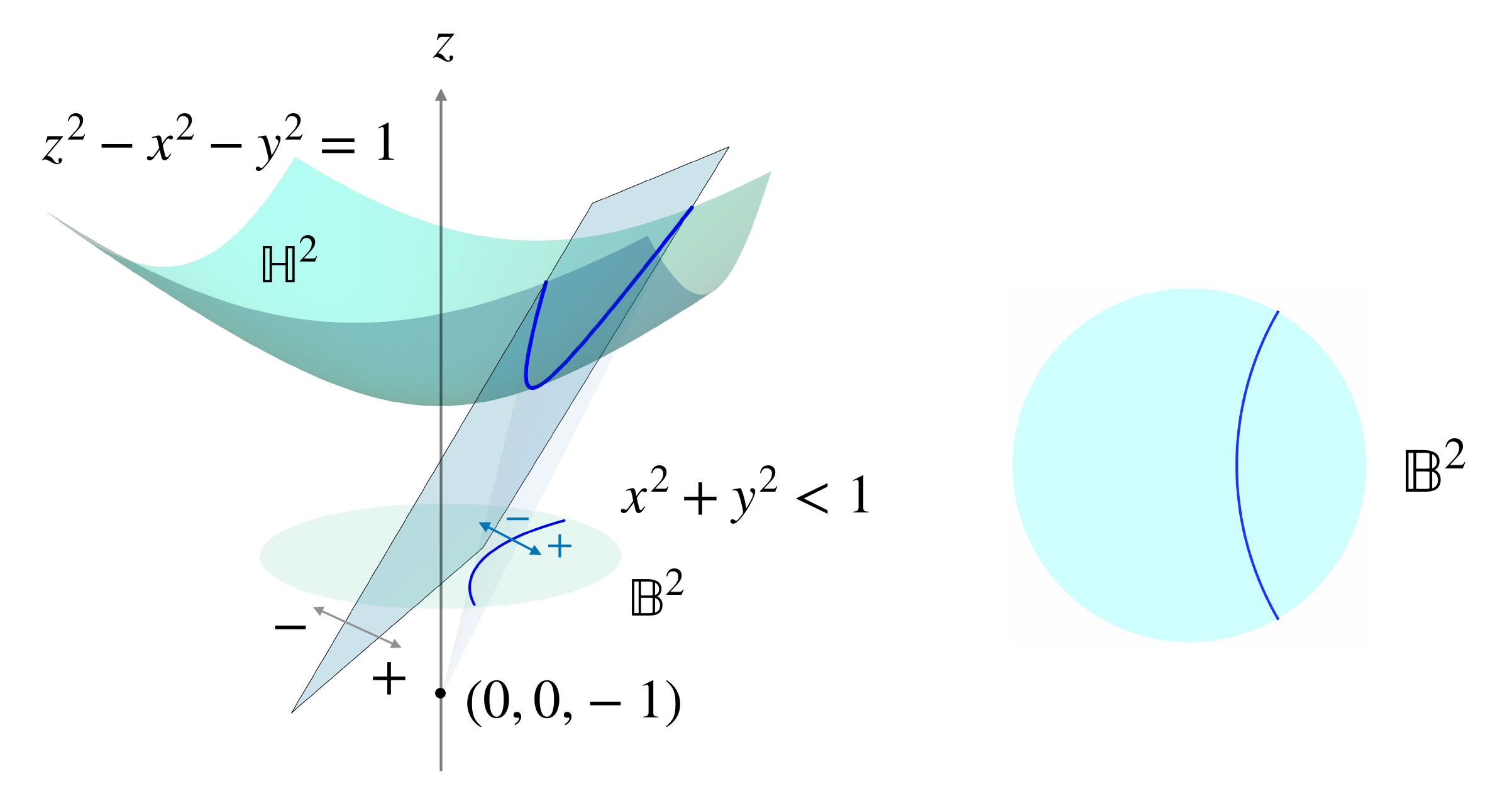}
\caption{A geodesic in the hyperboloid model of the hyperbolic plane $\Hyp^2$, obtained as the intersection of a plane with the hyperboloid. The corresponding geodesic on the Poincare disc model $\mathbb B^2$ of the hyperbolic plane is obtained by stereographic projection from the point $(0,0,-1)$.}
\label{hyperbolic_fig}
\end{figure}

\begin{example} \normalfont
We now present a kernel that does not admit a PD decomposition. Let $B$ be a real reflexive Banach space whose norm does not arise from an inner product (for instance, an $L^p$ space with $1<p<\infty$, $p\neq 2$), $B^\ast $ be its dual. Let $X= B\times B^\ast $, and let the kernel $k: X\times X\to \R$ be given by
$$k(x,\varphi; y,\psi):= \varphi(y)+\psi(x).$$
Then $k$ has no PD decomposition.  This example was first presented in {\citet[Page 243]{schwartz_sous-espaces_1964}}. A self-contained proof can be found in {\citet[Theorem 2.2]{alpay_remarks_1991}}.
\end{example}

\section{Group Actions and Invariant Kernels}\label{sta_ker_sec}
To have at our disposal additional structure to address Problem \ref{problem}, we shall
restrict our attention to kernels that possess certain symmetries, which we shall refer to as {\it invariant kernels}. 
In the rest of the paper, $G$ will be a locally compact (always assumed Hausdorff) topological group, acting on the left on the set $X$.
\begin{definition}
    A kernel $k:X\times X\to \C$ is called invariant with respect to the action of $G$ if
    $$k(g\cdot x, g\cdot y) = k(x,y)$$
    for all $g\in G$ and all $x,y\in X$, $\cdot$ denotes a left action of the elements of $G$ on the elements of $X$.
\end{definition}
{\begin{remark} \normalfont
    For notational convenience, we will, at times, drop the $\cdot$ all together and write $g\cdot x=gx$.
\end{remark}}
PD invariant kernels were first studied by \cite{yaglom_second-order_1961}, and more recently by \cite{azangulov_stationary_2023, azangulov_stationary_2023-1} in the context of Gaussian processes, {and by \cite{da_costa_invariant_2023}.}
\begin{example} \normalfont\label{trans_ker_ex}
Taking $X = G = \R^n$ equipped with addition and acting on itself, an invariant kernel on $\R^n$ is a {\it translation-invariant} or {\it stationary} kernel:
$$k(x+g,y+g) = k(x,y)$$
for all $x,y,g\in \R^n$. Standard invariant theory arguments show that translation-invariant kernels are exactly those that depend only on the difference $x-y$, i.e. kernels $k$ on $\R^n$ for which there is a function $f$ on $\R^n$ such that
$$k(x,y) = f(x-y)$$
for all $x,y\in \R^n$.
\end{example}
In this paper, we shall study a generalization of this example, namely when the action of $G$ is transitive on $X$. In that case, if we pick a distinguished element $o$ of $X$, the orbit-stabilizer theorem gives us the bijection of sets
\begin{equation*}
\begin{aligned}
G/H &\to X \\
gH &\mapsto g\cdot o
\end{aligned}
\end{equation*}
where $H= \Stab(o)$. In what follows, we will write $X$ as $G/H$, where $H$ is a subgroup of $G$. This notation has the advantage of specifying both the group $G$ acting on $X$ and the action (the canonical left action of $G$ on $G/H$), so from here on, we will call an invariant kernel on $G/H$ one that is invariant with respect to that action. We will further make the assumption that $H$ is closed, granting us with a locally compact Hausdorff quotient topology for $X=G/H$, and making the action of $G$ on $X$ proper (see {\citet[Section 2.6]{folland_course_2015}} for the theory of such homogeneous spaces).\par
Invariant kernels under transitive actions are particularly interesting for us because the symmetries they respect allow us to view them as functions of a single variable instead of two, as in Example \ref{trans_ker_ex}. To see this, let us first give the analogues to Definitions \ref{kernel_def} and \ref{decomposition_def} for functions defined on the relevant spaces.

\begin{definition} Write $H\backslash G/H := \{HgH: g\in G\}$ for the double coset space. A complex-valued function $f:H\backslash G/H \to \C$ is said to be positive definite (PD) if for all $N\in\N$ and $g_1,\dots,g_N\in G$, the matrix $$\big(f(Hg_i^{-1}g_jH)\big)_{i,j}$$
is positive semidefinite. $f$ is said to have a positive (PD) decomposition if it can be written as
$$f = f_+-f_-$$
where $f_+$ and $f_-$ are PD functions on $H\backslash G/H$.
\end{definition}
In particular, when $H = \{e\}$, $f:G\to \C$ is PD if the matrices $\big(f(g_i^{-1}g_j)\big)_{i,j}$ are positive semidefinite for all $N$ and all $g_i, g _j \in G$. Functions on $H\backslash G/H$ should be considered as $H$-invariant functions on $G/H$. The relevance of double coset spaces in the kernel context will become clear later in Proposition \ref{correspondence_prop}.

\begin{figure}
\centering
\includegraphics[width=1.0\linewidth]{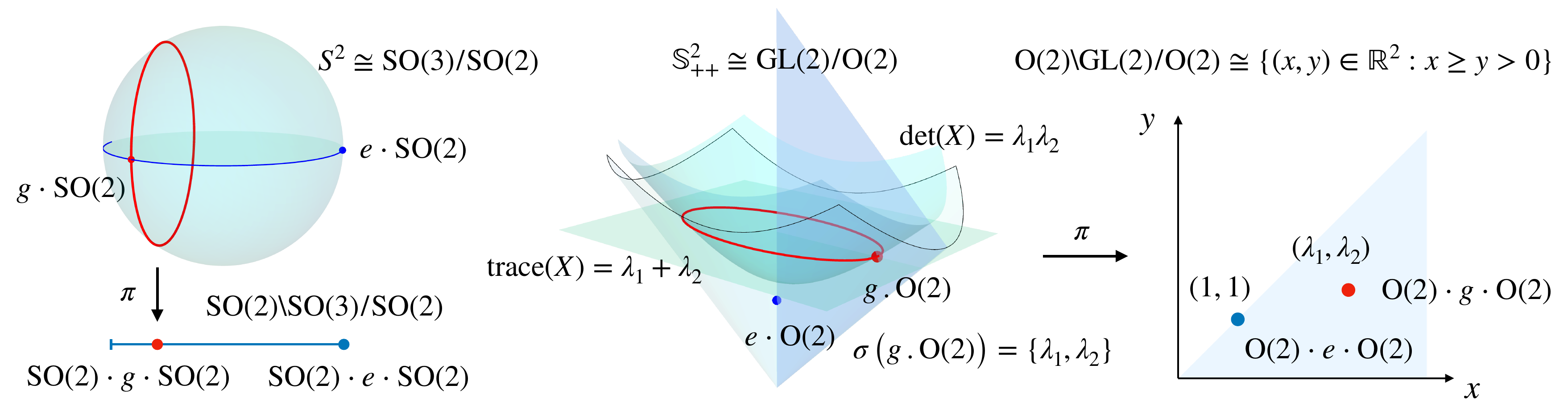}
\caption{A visualisation of the projections to the double coset spaces $\pi: S^2\cong SO(3)/SO(2)\to SO(2)\backslash SO(3)/SO(2)$ and $\pi: \mathbb S_{++}^2\cong GL(2)/O(2)\to O(2)\backslash GL(2)/O(2)$. In each case, we have in red an orbit $HgH$. In the case of $S^2$, it is obtained by rotating $gSO(2)$ around a horizontal axis. In the case of $\mathbb S_{++}^2$, it is the set of matrices with eigenvalues equal to the ones of $gO(2)$, obtained as the intersection of the constant-trace plane and the constant-determinant surface. In each case, we have in dark blue a set transversal to such orbits, and homeomorphic to $H\backslash G/H$.}
\label{double_coset_fig}
\end{figure}

\begin{example} \normalfont\label{doub_coset_sphere_ex}
    The $n$-dimensional sphere $S^n$ may be viewed as the coset space
    $$S^n \cong SO(n+1)/SO(n).$$
    Moreover, viewed as a submanifold $S^n\subset \R^{n+1}$, $SO(n+1)$ acts transitively on $S^n$ by rotations. Picking $o=(1,0,\dots,0)\in S^n$, the stabilizer of $o$ is the subgroup of $SO(n+1)$ leaving the first coordinate untouched, which is isomorphic to $SO(n)$. Now the double coset space $SO(n)\backslash SO(n+1) /SO(n)$ is obtained by quotienting out each orbit (see Figure \ref{double_coset_fig}), so we see
    $$SO(n)\backslash SO(n+1) /SO(n) \cong [-\pi,\pi].$$
    Functions on $SO(n)\backslash SO(n+1) /SO(n)$ are functions of the geodesic distance on $S^n$, or equivalently of the inner product of $\R^{n+1}$ restricted to $S^n$.
\end{example}
\begin{example} \normalfont\label{doub_cose_spd_ex}
    The space of $n\times n$ real symmetric positive definite matrices $\mathbb S_{++}^n$ may be viewed as the coset space
    $$\mathbb S_{++}^n \cong GL(n)/O(n).$$
    Indeed, $GL(n)$ acts transitively on $\mathbb S_{++}^n$ by $G\cdot X := GXG^{\top}$
    for $G\in GL(n)$ and $X\in \mathbb S_{++}^n$. The stabilizer of the the identity $I$ consists of the matrices $G$ such that $GG^{\top}=I$, i.e. $\Stab(I) = O(n)$. Now, for $X\in \mathbb S_{++}^n$, $O(n)XO(n) \in O(n)\backslash GL(n)/O(n)$ consists of $\{HXH^{\top} : H\in O(n)\}$, which is fully determined by the $n$ eigenvalues of $X$. So
    $$O(n)\backslash GL(n) /O(n) \cong \{(\lambda_1,\dots,\lambda_n): \lambda_i>0\; \forall\; i\}/\sim$$ %\cong \{S\subset \R_{>0}: |S|\leq n\}$$
    where $(\lambda_1,\dots,\lambda_n)\sim (\lambda_1',\dots,\lambda_n')$ if and only if $\{\lambda_1,\dots,\lambda_n\}=\{\lambda_1',\dots,\lambda_n'\}$ (see Figure \ref{double_coset_fig}). Functions on $O(n)\backslash GL(n) /O(n)$ are functions of the eigenvalues on $\mathbb S_{++}^n$.
\end{example}
Let us generalize the concept of Hermitian functions to complex-valued functions on double coset spaces:
\begin{definition}
    $f:H\backslash G/H \to \C$ is called Hermitian if
    $$f(Hg^{-1}H)= \overline{f(HgH)}$$
    for all $g\in G$.
\end{definition}
Since $G$ is a topological group, $G/H$ and $H\backslash G/H$ are equipped with natural quotient topologies. In the rest of the paper, we shall make the following continuity assumption: we will consider continuous kernels $k$, and ask whether they admit PD decompositions into continuous PD kernels $k= k_+-k_-$. We call such a decomposition a {\it continuous PD decomposition}.
\begin{remark} \normalfont
    The continuity assumption is non-trivial: see {\citet[Section 3]{schaback_approximation_2003}} for a construction of an everywhere discontinuous PD kernel. Note, however, that at least for functions on groups, a continuous function on $G$ has a continuous PD decomposition if and only if it has a PD decomposition \cite[Corollary 2.2.2]{kaniuth_fourier_2018}.
\end{remark}
The reason for considering functions on the \emph{double} coset space $H\backslash G/H$ as opposed to the regular coset space $G/H$ is made apparent in the following important result.
\begin{proposition}
\label{correspondence_prop}
    There is a bijection between the continuous Hermitian functions on $H\backslash G/H$ and the continuous invariant kernels on $G/H$. This bijection restricts to:
\begin{enumerate}
    \item A bijection between the continuous Hermitian PD functions and the continuous PD invariant kernels.
    \item An injection between the continuous Hermitian PD functions with continuous PD decompositions and the continuous kernels with continuous PD decompositions.
\end{enumerate}
\end{proposition}
\begin{proof}
    Given a continuous invariant kernel $k:G/H \times G/H \to \C$, define the function $f:H\backslash G/H \to \C$ by
    $$f(HgH):= k(gH,eH)$$
for all $g\in G$, where $e$ is the identity of $G$. Note that
$$f(Hg^{-1}H) = k(g^{-1}H,e) = k(g^{-1}H,eH) = k(eH,gH) = \overline{k(gH,eH)} = \overline{f(HgH)}$$
i.e. if $k$ is Hermitian then $f$ is necessarily Hermitian. Finally, the continuity of $f$ follows from the continuity of $k$. Conversely, given a Hermitian function $f: H\backslash G/H \to \C$, define the invariant kernel $k: G/H\times G/H \to \C$ by
$$k(gH,g'H) := f(Hg^{-1}g'H)$$
for all $g,g'\in G$ where the continuity of and the Hermitian character of $f$ imply the same properties for $k$. \par
These correspondences are inverses of each other and thus provide a bijection. Moreover, we can see that $k$ is PD if and only if $f$ is PD. Also, this correspondence is linear, which implies that if $f$ admits a continuous PD decomposition, so does $k$.
\end{proof}

\begin{remark} 
\normalfont
    In the case $H = \{e\}$, the correspondence above is simply the bijection between Hermitian functions on $G$ and invariant kernels on $G$. Example \ref{trans_ker_ex} is an instance of this case.
\end{remark}
\begin{remark} \normalfont\label{stat_rmk}
    Note that in the second part of Proposition \ref{correspondence_prop}, we have only proved the injectivity of the correspondence between continuous Hermitian PD functions with continuous PD decompositions and the continuous kernels with continuous PD decompositions. The difficulty in proving surjectivity is that $k$ may not admit a continuous PD decomposition $k= k_+-k_-$ with $k_+$ and $k_-$ invariant. In Appendix \hyperref[averaging_app]{B} we provide a partial solution by showing that the correspondence in part 2 is indeed surjective when $G$ is compact. In any case, the injectivity in this part will be enough for us to provide sufficient conditions for a kernel $k$ to have a continuous PD decomposition.
\end{remark}
Being able to handle invariant kernels as functions will allow us to leverage tools from harmonic analysis and representation theory to address Problem \ref{problem}. \par

In the sequel, we shall use the following spaces in relation with the locally compact group $G$ and the closed subgroup $H$:
$$
\begin{aligned}
    B_+(G) &:= \{\text{continuous PD functions } G\to \C\} \\
    B_\R(G) &:= {\spn_{\R}(B_+(G))} \\
    &= \{\text{continuous functions with continuous PD decompositions } G\to \C\} \\
    B_\C(G) &:= {\spn_{\C}(B_+(G))}\\
\end{aligned}
$$
and, more generally,
$$
\begin{aligned}
    B_+(H\backslash G/H) &:= \{\text{PD functions } H\backslash G/H\to \C\} \\
    B_\R(H\backslash G/H) &:= {\spn_{\R}(B_+(H\backslash G/H))} \\
    &= \{\text{continuous functions with continuous PD decompositions } H\backslash G/H\to \C\} \\
    B_\C(H\backslash G/H) &:= {\spn_{\C}(B_+(H\backslash G/H))}.
\end{aligned}
$$
\begin{remark} \normalfont
    While $B_\R(G)$ and $B_\R(H\backslash G/H)$ are defined as the real span of $B_+(G)$ and $B_+(H\backslash G/H)$, note that the latter contain complex-valued functions that are not real-valued, hence so do the former.
\end{remark}
The space $B_\C(G)$ is called the {\it Fourier-Stieltjes algebra} of $G$ in the literature and has been studied thoroughly; see, for example, \cite{kaniuth_fourier_2018}. We will present and extend some of these results to $B_\C(H\backslash G/H)$. \par
Motivated by Proposition \ref{correspondence_prop}, our goal is to characterize $B_\R(H\backslash G/H)$ and consider Hermitian elements in this space. In fact, since we are interested in real-valued kernels and hence in real-valued functions, the following proposition shows that it suffices to characterize $B_\C(H\backslash G/H)$ and consider real-valued even ($f(HgH)=f(Hg^{-1}H)$ for all $g$) elements in this space.

\begin{proposition}\label{real_val_prop}
    If $f\in B_\C(G)$ or $f\in B_\C(H\backslash G/H)$ is real-valued, then $f\in B_\R(G)$ or $f\in B_\R(H\backslash G/H)$ respectively.
\end{proposition}
\begin{proof}
    Observe that if $f\in B_+(G)$, then $\Re f= \frac{f+\overline f}{2} \in B_+(G)$, and $\Im f= \frac{f-\overline f}{2} \in B_\R(G)$. \par
    If $f\in B_\C(G)$ we can write
    $$f = f_+-f_-+if_{+i}-if_{-i}$$
    with $f_+,f_-,f_{+i},f_{-i}\in B_+(G)$. If $f$ is real-valued,
    $$
    \begin{aligned}
    f = \Re f &= \Re f_+-\Re f_-+\Re(if_{+i})-\Re(if_{-i}) = \Re f_+-\Re f_- -\Im f_{+i}+\Im f_{-i}
    \end{aligned}
    $$
so $f\in B_\R(G)$. The proof is identical for functions in $B_\C(H\backslash G/H)$.
\end{proof}

\section{Positive Decompositions for Invariant Kernels: Commutative Case}\label{com_sec}
In this section, we focus on invariant kernels that are invariant with respect to the action of an Abelian locally compact group $G$.  The commutativity hypothesis will simplify the theory considerably. \par
First, observe that in this case, any closed subgroup $H$ of $G$ is automatically a normal subgroup, so $X=G/H$ is itself a group. Equipped with the quotient topology, $G/H$ is actually a locally compact group \cite[Proposition 2.2]{folland_course_2015}. Therefore without loss of generality, we can set $X=G$. \par

In view of the results in the previous section, we shall focus in this case on $B_\C(G)$, which is a well-studied object in standard commutative abstract harmonic analysis. For the sake of completeness, we recall here a few facts in relation to this theory that are needed in our developments. We refer to \cite{folland_course_2015} for a more detailed exposition and proofs. {\citet[Chapter 4]{folland_course_2015}} is particularly relevant to this section.\par
$G$ has a {\it dual group}, $\widehat G$, which is defined as the group of continuous irreducible unitary representations of $G$ {(see Definition \ref{representation_def} in Appendix \hyperref[tech_sec]{C} where the concept of representations becomes much more central)}. $G$ being Abelian implies that irreducible unitary representations of $G$ are one-dimensional, so $\widehat G$ consists of the continuous group homomorphisms $\xi: G\to S^1\subset \C${, i.e. $\widehat G = \operatorname{Hom}(G,S^1)$}. Equipped with the open-compact topology, $\widehat G$ is itself a locally compact group. Moreover, we have a natural topological group isomorphism {$\widehat{\widehat{G}\,}\cong G$}. This is the {\it Pontryagin Duality Theorem}. 

Now define
$$
\begin{aligned}
    M_+(G) &:= \{\text{finite positive Radon measures on $G$}\} \\
    M_\R(G) &:= {\spn_\R(M_+(G))} \\
    &= \{\text{finite signed Radon measures on $G$}\} \\
    M_\C(G) &:= {\spn_\C(M_+(G))} \\
    &= \{\text{finite complex Radon measures on $G$}\}
\end{aligned}
$$
and $M_+(\widehat G)$, $M_\R(\widehat G)$, $M_\C(\widehat G)$ analogously. Importantly for us, we can define the {\it Fourier transform} $\F: M_\C(G) \to C{_\C}(\widehat G)$ by
$$\F\mu(\xi) = \int_G \overline{\xi(g)}d\mu(g)$$
for all $\xi\in \widehat G$. Here $C{_\C}(\widehat G)$ is the space of continuous complex-valued functions on $\widehat G$. Similarly, we can define the inverse Fourier transform $\F^{-1}: M_\C(\widehat G) \to C{_\C}(G)$ by
$$\F^{-1}\nu(g) = \int_G\xi(g)d\nu(\xi)$$
for all $g\in G$. $\F$ and $\F^{-1}$ are injective linear maps. \par
Fourier transforms are often thought of as acting on functions. Our description of Fourier transforms encompasses this: note that $G$ comes equipped with a Haar measure $\mu$. This is a translation-invariant Radon measure on $G$, which is unique up to multiplication by a positive constant. Then, for $f$ a $\mu$-integrable function on $G$, $f\mu \in M_\C(G)$, so we can define the Fourier transform of $f$ as $\F (f\mu)$. \par
\begin{remark} \normalfont
    $\F$ and $\F^{-1}$ will only be inverse to each other with the right choice of Haar measures on $G$ and $\widehat G$. Indeed, as described above, the actions of the operators $\F$, $\F^{-1}$ on $L^1(G)$, $L^1(\widehat G)$ depend on the choice of Haar measures. See {\citet[Theorem 4.22]{folland_course_2015}}.
\end{remark}
\begin{example} \normalfont\label{real_fourier_ex}
    When $G= \R^n$, we have $\widehat G \cong \R^n$, and the Fourier transform can be written as
    $$\F\mu(\xi) = \int_{\R^n}e^{-i\langle \xi,x\rangle} d\mu(x)$$
    for $\mu\in M_\C(\R^n)$ and $\xi\in \R^n$. The inverse Fourier transform can be written as
    $$\F^{-1}\nu(x) = \int_{\R^n} e^{i\langle \xi, x\rangle} d\nu(\xi)$$
    for $\nu\in M_\C(\R^n)$ and $x\in \R^n$.
\end{example}
\begin{example} \normalfont\label{circle_fourier_ex}
    When $G= S^1$, we have $\widehat G \cong \Z$, and the Fourier transform corresponds to the discrete Fourier transform, and can be written as
    $$\F\mu(k) = \int_{S^1}e^{-ik[x]} d\mu(x)$$
    for $\mu\in M_\C(S^1)$ and $k\in \Z$, where $[x]\in [-\pi,\pi)$ is the argument of $x$. The inverse Fourier transform can be written as
    $$\F\nu(x) = \int_{S^1}e^{ik[x]} d\nu(x)$$
    for $\nu \in M_\C(S^1)$ and $x \in S^1$.
\end{example}
We are now able to state the key theorem of this section \cite[Theorem 4.19]{folland_course_2015}.
\begin{theorem}[Bochner's Theorem]
    $f\in B_+(G)$ is PD if and only if there is $\nu\in M_+(\widehat G)$ such that $f = \F^{-1} \nu$. Moreover, such $\nu$ is unique.
\end{theorem}
In other words, $\F^{-1}$ is a linear isomorphism between $M_+(\widehat G)$ and $B_+(G)$ (which are convex cones, not vector spaces). Therefore, tensoring by $\R$ and then by $\C$ in Bochner's Theorem, we also get that $\F^{-1}$ is a linear isomorphism between $M_\R(\widehat G)$, $M_\C(\widehat G)$ and $B_\R(G)$, $B_\C(G)$, respectively. In other words, we have the commutative diagram
\[\begin{tikzcd}
	{M_+(\widehat G)} & {B_+(G)} \\
	{M_\R(\widehat G)} & {B_\R(G)} \\
	{M_\C(\widehat G)} & {B_\C(G)} \\
	& {C{_\C}(G)}
	\arrow[hook', from=1-2, to=2-2]
	\arrow[hook', from=2-2, to=3-2]
	\arrow[hook', from=3-2, to=4-2]
	\arrow["{\F^{-1}}"', from=3-1, to=4-2]
	\arrow[tail reversed, "\sim", from=3-1, to=3-2]
	\arrow[tail reversed, "\sim", from=2-1, to=2-2]
	\arrow[tail reversed, "\sim", from=1-1, to=1-2]
	\arrow[hook', from=1-1, to=2-1]
	\arrow[hook', from=2-1, to=3-1]
\end{tikzcd}\]
where $\sim$ denote linear isomorphisms and $\hookrightarrow$ denote the inclusions. In other words, PD decomposable functions on $G$ are the functions that are the inverse Fourier transforms of finite signed measures on $\widehat G$. This description of $B_\R(G)$ and $B_\C(G)$ is in some sense the best we have, as being the inverse Fourier transform of a finite measure does not appear to be a reducible property. However, what we can do is provide \emph{sufficient} conditions for a function to be in $B_\R(G)$ and $B_\C(G)$. \par
First, let us state two important results.
\begin{theorem}\label{inf_ab_thm}
    The map $\F^{-1}:M(\widehat G) \to C{_\C}(G)$ is surjective if and only if $G$ is finite. So, $B_\C(G) = C{_\C}(G)$ if and only if $G$ is finite.
\end{theorem}
See \cite{graham_behavior_1979}. We showed the \emph{if} direction in Example \ref{finite_ex}, and this proposition provides a converse. It also provides a glimpse at the fact that characterizing $B_\C(G)$ is a deep problem. \par
We have, however, the following positive result: 
\begin{theorem}\label{dense_ab_thm}
    $B_\C(G)\cap C_c{_\C}(G)$ is dense in $C_c{_\C}(G)$, the space of compactly supported continuous functions, with the uniform norm. As a consequence, $B_\C(G) \cap L^p(G)$ is dense in $L^p(G)$ for all $1\leq p < \infty$.
\end{theorem}
See {\citet[Proposition 3.33]{folland_course_2015}}. Theorem \ref{dense_nab_thm} below is the general version of this result for the non-commutative case. \par
For the commutative case, there are, in practice, two examples of interest: $\R$ and $S^1$. This is reflected by the fact that any Abelian Lie group $G$ is a product of these. In Example \ref{real_fourier_ex} and Example \ref{circle_fourier_ex}, we have already given expressions for the Fourier transform in these cases. Let us study the implications in each case.
\begin{example} \normalfont\label{real_stieltjes_ex}
    $G=\R^n$, $\widehat G = \R^n$. In this case, it is well-known that $\F$ is an automorphism of the Schwartz space $\S (\R^n)$. So, in particular, $\S (\R^n)\subset B_\C(\R^n)$, where we recall that
    $$\S (\R^n) := \left\{f\in C^\infty_{{\C}}(\R^n) : \sup _{x \in \mathbb{R}^n}\left|x^\alpha\left(D^\beta f\right)(x)\right|<\infty \; \forall \; \alpha,\beta\in\N^n\right\}.$$
    In other words, $\S (\R^n)$ is the space of complex-valued smooth functions on $\R^n$ whose values and derivatives decay faster at infinity than any reciprocal of a polynomial.
\end{example}
\begin{example} \normalfont\label{circle_stieltjes_ex}
    We can view $S^1$ as $G= \R/\Z$, and we have $\widehat G = \Z$. As described in Example \ref{circle_fourier_ex}, the Fourier transform in this case is the discrete Fourier transform. We have that $M_\C(\Z)$ is $l^1_\C(\Z)$, and $B_\C(\R/\Z)$ is the space of absolutely continuous Fourier series. This is called the Wiener algebra, and we have that the Hölder continuous functions $C^\alpha_{{\C}}(\R/\Z) \subset B_\C(\R/\Z)$ for $\alpha > 1/2$. This is Bernstein's Theorem, see for example {\citet[Theorem 6.3]{katznelson_introduction_2004}}. Recall that for $\alpha< 1$,
    $$C^\alpha_{{\C}}(\R/\Z):= \left\{f\in C_{{\C}}(\R/\Z) :  \sup_{x,y\in\R/\Z,\; x\neq y}\frac{|f(x)-f(y)|}{|x-y|^\alpha}< \infty \right\},$$
    where the absolute value $|x-y|$ in the denominator is taken modulo 1.
\end{example}

\begin{example}\label{gk_circ_ex} \normalfont
    The Gaussian kernel on the circle $S^1$
    $$k(x,y) = \exp(-\lambda|x-y|^2)$$
    for all $x,y\in \R/\Z \cong S^1$, where $\lambda >0$ and the absolute value is taken modulo 1, is not PD for any $\lambda >0$ (see \cite{da_costa_gaussian_2023}). However, note that it is invariant, and under Proposition \ref{correspondence_prop} corresponds to the Gaussian function $f(x) = \exp(-\lambda|x|^2)$ for all $x\in \R/\Z\cong S^1$, where the absolute value is taken modulo 1. $f$ is Lipschitz, so by Example \ref{circle_stieltjes_ex}, $f\in B_\C(S^1)$, and since it is real-valued, $f\in B_\R(S^1)$ by Proposition \ref{real_val_prop}. So $k$, despite not being PD, has a PD decomposition.
\end{example}
From Example \ref{gk_circ_ex}, we obtain the following result.
\begin{proposition}
    The Gaussian kernel $k = \exp(-\lambda d(\cdot,\cdot)^2)$ has a PD decomposition on any Abelian Lie group (with appropriate choice of the metric $d$) for any $\lambda >0$.
\end{proposition}
\begin{proof}
    Any Abelian Lie group $G$ is a product of copies of $\R$ and $S^1$, say $G \cong \R^n\times (S^1)^m$. Taking the product Riemannian metric on the right-hand side of this induced by the standard Riemannian metrics on $\R$ and $S^1$, we can write the Gaussian kernel $k$ on $G$ as
    $$k(x,y) = \exp(-\lambda d(x,y)^2) =\prod_{i=1}^{n+m}\exp(-\lambda|x_i-y_i|^2)$$
    for all $x,y\in G\cong \R^n\times (S^1)^m$, where the last $m$ absolute values are taken modulo 1. By positive definiteness of the Gaussian kernel on $\R$ and PD decomposability of the Gaussian kernel on $S^1$ (Example \ref{gk_circ_ex}), $k$ is a product of PD decomposable kernels, so has a PD decomposition by Example \ref{product_ex}.
\end{proof}
We conclude this section with the following example.
\begin{example} \normalfont
    Suppose $\tilde f\in C^\alpha_{{\C}}(\R)$ for some $\alpha>1/2$ and that $\tilde f$ is periodic, with period $T$, say. Then, if $\pi: \R \to \R/T\Z$ is the projection, there is a function $f\in C^\alpha_{{\C}} (\R/T\Z)$ such that $f\circ\pi =\tilde f$. Then, $f\in B_\R(\R/T\Z)$, assuming for simplicity that $\tilde f$ is real-valued. Writing its PD decomposition $f = f_+-f_-$, we have that $\tilde f = f_+\circ \pi -f_-\circ\pi$ is a PD decomposition, so that $\tilde f\in B_\R(\R)$. \par
    PD decompositions behave well with respect to quotients. We will prove a general version of this phenomenon, as well as its converse, in Section \ref{non_com_sec}.
\end{example}

\section{Positive Decompositions for Invariant Kernels: Non-Commutative Case}\label{non_com_sec}
Armed with an understanding of the commutative case, we are now in a position to tackle the non-commutative case. Recall that we aim to characterize sufficient conditions for an invariant kernel to admit a PD decomposition, which by Proposition \ref{correspondence_prop} and Proposition \ref{real_val_prop}, comes down to characterizing sufficient conditions for a function to belong to $B_\C(H\backslash G/H)$. Here the only assumptions are that $G$ is a locally compact group and $H$ is a closed subgroup of $G$. \par
We begin by observing that, just as in the commutative case (Theorem \ref{dense_ab_thm}), PD decomposable functions on $G$ are plentiful and are, in fact, dense in the space of compactly supported continuous functions:
\begin{theorem}\label{dense_nab_thm}
    $B_\C(G)\cap C_c{_\C}(G)$ is dense in $C_c{_\C}(G)$, the space of compactly supported continuous functions, with the uniform norm. As a consequence, $B_\C(G) \cap L^p(G)$ is dense in $L^p(G)$ for all $1\leq p < \infty$.    
\end{theorem}
See {\citet[Corollary 2.3.5]{kaniuth_fourier_2018}} for a proof. \par
To obtain more specific descriptions of $B_\C(G)$ and $B_\C(H\backslash G/H)$, there are multiple different possible approaches. Firstly, Fourier transforms can be defined on general compact groups $G$ analogously to the circle: they can be described as Fourier series, see {\citet[Chapter 5]{folland_course_2015}}. The non-compact case is more difficult. Additional regularity assumptions must be placed on the group, for example, that it is of type I, see for instance \cite{yaglom_second-order_1961}. This latter reference provides Bochner-like theorems for the homogeneous spaces of type I groups, which include symmetric spaces. The most general version of Bochner's Theorem is known as the Bochner-Godement Theorem, see {\citet[Theorem 3.1]{faraut_distances_1974}}, which describes PD $H$-invariant functions on $G/H$, i.e. PD functions on $H\backslash G/H$, with some additional mild regularity assumptions on $G/H$. \par
All of these references are quite technical. In this paper, we will take a different route. Describing a general Bochner Theorem for non-Abelian groups will not get us much closer to characterizing $B_\C(H\backslash G/H)$. What we are really after are concrete sufficient conditions for a function to belong to $B_\C(H\backslash G/H)$ on spaces of interest, for instance, when $G$ is a Lie group and $H$ is a Lie subgroup. Therefore, we will proceed as follows. \par
As exemplified by Theorem \ref{dense_nab_thm}, $B_\C(G)$ has been more extensively studied than the more general $B_\C(H\backslash G/H)$. We will, therefore, start by describing how the latter relates to the former. After this, we will be in a position in Section \ref{suf_cond_sec} to describe sufficient conditions for a function to belong to $B_\C(H\backslash G/H)$ for general classes of Lie groups $G$ and Lie subgroups $H$.

Suppose that $f\in C{_\C}(H\backslash G/H)$. If $f\in B_\R(H\backslash G/H)$ and $\pi: G\to H\backslash G/H$ is the projection, write
$$f = f_+-f_-$$
with $f_+,f_-\in B_+(H\backslash G/H)$. Then
$$f \circ \pi = f_+\circ\pi-f_-\circ\pi$$
with $f_+\circ\pi,f_-\circ\pi\in B_+(H\backslash G/H)$, so $f\circ\pi\in B_\R(G)$. \par
If we can characterize $B_\R(G)$ well, then it is the converse we would like to show: that if $f\in C{_\C}(H\backslash G/H)$ and $f\circ\pi\in B_\R(G)$, then $f \in B_\R(H\backslash G/H)$. This is \emph{not} obvious: from the fact that $f$ admits a PD decomposition with functions on $G$ and is invariant on double cosets, we would like to show that we can choose a PD decomposition for $f$ for which the terms are themselves invariant on double cosets. \par
One idea may be, given $f\in C{_\C}(H\backslash G/H)$ with $f\circ\pi\in B_\R(G)$ and a PD decomposition $f\circ\pi = \tilde f_+- \tilde f_-$, to ``average out" values of $\tilde f_+$ and $\tilde f_-$ along each double coset, in order to find PD $f_+, f_- \in C{_\C}(H\backslash G/H)$ such that $f= f_+- f_-$. This argument works if $H$ is compact. %and we make this proof precise in Appendix \hyperref[averaging_app]{B}.

{
For the general case, the situation is more difficult, as we cannot simply average over the non-compact double cosets. We will only need the case $H$ compact in Section \ref{suf_cond_sec}, thus we only prove this case here. Moreover, the case $H$ compact is sufficient for Riemannian homogeneous spaces by the Myers-Steenrod Theorem \citep[Theorem 2.35]{gallot_riemannian_2004}. The reader may refer to Appendix \hyperref[tech_sec]{C} for a full proof of Theorem \ref{lie_to_homog_thm}, using the language of representation theory.}

In the statement and the proof of Theorem \ref{lie_to_homog_thm} we will view $C{_\C}(H\backslash G/H)$ as a subset of $C{_\C}(G)$, which will simplify notation.

\begin{restatable}{theorem}{lietohomogthm}\label{lie_to_homog_thm} We have
    \begin{enumerate}
    \item $B_\R(G)\cap C_{{\C}}(H\backslash G/H) = B_\R(H\backslash G/H),$
    \item $B_\C(G)\cap C{_\C}(H\backslash G/H) = B_\C(H\backslash G/H).$
    \end{enumerate}
\end{restatable}

\noindent\textbf{Proof when $H$ compact\ } $\supseteq\,:$ For part 1, if $f\in B_\R(H\backslash G/H)$ then $f\in C{_\C}(H\backslash G/H)$ by definition, and $f$ has a PD decomposition, so $f\in B_\R(G)$ (or see the argument at the beginning of this section). Part 2 is similar.

$\subseteq\,:$ For $g\in G$, define $L_g$ and $R_g$ to be the operators acting on functions $f:G\to\C$ by
$$
    L_gf(a) =f(ga)
$$
and
$$
    R_gf(a) =f(ag)
$$
for all $a\in G$. Suppose $f\in B_\R(G)\cap C{_\C}(H\backslash G/H)$. Then, by invariance of $f$ on double cosets,
$$f(g) = \frac{1}{\nu(H)^2}\int_H\int_H R_{h'}L_hf(g) d\nu(h)d\nu(h')$$
for $g\in G$, where $\nu$ is a translation bi-invariant Haar measure on $H$, which exists and is finite by compactness of $H$. So writing the PD decomposition of $f$ as $f = f_+-f_-$, we get
$$f = \frac{1}{\nu(H)^2}\int_H\int_H R_{h'}L_hf_+ d\nu(h)d\nu(h')- \frac{1}{\nu(H)^2}\int_H\int_H R_{h'}L_hf_- d\nu(h)d\nu(h').$$
Both the first term and minus the second term are PD, and they are invariant on double cosets by bi-invariance of the Haar measures. Continuity of the terms is easy to check, and the proof for $B_\C(G)$ and $B_\C(H\backslash G/H)$ is similar. \hfill\BlackBox \\

\subsection{Geometry and sufficient conditions for positive decomposability}\label{suf_cond_sec}
Assuming some geometry on $G$ and $H$, we are now in a position to describe sufficient conditions for a function to be in $B_\C(H\backslash G/H)$. Specifically, we assume in this section that $G$ is a connected Lie group. Furthermore, we assume that $G$ is unimodular, meaning it possesses a Haar measure that is bi-invariant. This includes all Abelian and all compact Lie groups, but also many non-compact Lie groups of interest, such as $GL_\R(n)$ and $GL_\C(n)$. \par
Then, we have the following powerful characterization.
\begin{theorem}\label{bernstein_thm}
Fix a left-invariant Riemannian metric on the connected unimodular Lie group $G$, and let $\Delta$ be the Laplace-Beltrami operator with respect to this metric. Suppose $f\in C^\infty_{{\C}}(G)$ is such that 
$$\Delta^kf = \underbrace{\Delta\dots\cdot\Delta}_{\text{k times}}f$$
is $L^2$-integrable with respect to the Haar measure on $G$ for any integer $k\geq 0$. Then $f\in B_\C(G)$.
\end{theorem}
\begin{proof}
When $G$ is compact, the assumption can be reduced to $f\in C^\infty_{{\C}}(G)$, and the result follows from \cite{sugiura_fourier_1971}, or from \cite{gaudry_bernsteins_1986}. For non-compact $G$, the result is a special case of {\citet[Theorem 1]{gaudry_heat_1990}}, which proves the inclusion of certain Besov spaces in Fourier algebras. See also the remark after {\citet[Equation 12]{gaudry_heat_1990}}, as well as the containment of the Fourier algebra $A_2(G)$ in the Fourier-Stieltjes algebra, for example in {\citet[Proposition 2.3.3]{kaniuth_fourier_2018}}. We should note that the description of the Laplacian in \cite{gaudry_heat_1990} as the sum of the squares of vector fields generating the Lie algebra $\g$ of $G$ is non-standard, but we prove in Appendix \hyperref[laplace_app]{D}, that the Laplace-Beltrami operator arising from a left-invariant metric on a unimodular Lie group can be formulated in these terms.
\end{proof}
\begin{remark} \normalfont
    For simplicity, Theorem \ref{bernstein_thm} relaxes the results from the works referred to in its proof. For instance, by \cite{gaudry_bernsteins_1986} we can show that when $G$ is compact, $C^k_{{\C}}(G)\subset B_\C(G)$ for $k>n/2$, where $n = \dim G$. This still does not capture the full strength of the result in this reference, which is a generalization of Bernstein's Theorem (see Example \ref{circle_stieltjes_ex}). On the other hand, when $G$ is non-compact, \cite{gaudry_heat_1990} shows, in particular, that it suffices that $\Delta^k f$ is weakly in $L^2$ for $k \leq (n^2+n)/4$ (see also {\citet[Equation 2.6]{varopoulos_analysis_1988}}).
\end{remark}
\begin{remark} \normalfont
    Compare this with Example \ref{circle_stieltjes_ex}. In both cases, for a function $f$ to belong to $B_\C(G)$ it suffices that
    \begin{enumerate}
        \item $f$ is smooth enough,
        \item $f$ and its derivatives decay fast enough at infinity.
    \end{enumerate}
\end{remark}
We would now like to generalize Theorem \ref{bernstein_thm} from $B_\C(G)$ to $B_\C(H\backslash G/H)$ with the help of Theorem \ref{lie_to_homog_thm}. For this, we assume $H$ is a compact Lie subgroup of $G$. Then, $G/H$ is a homogeneous space, which inherits the structure of a manifold from the one on $G$. Moreover, observe that if the left-invariant Riemannian metric on $G$ is also right-$H$-invariant, then it gives rise to a natural Riemannian metric on the homogeneous space $G/H$, making the projection map $\pi_R:G\to G/H$ into a Riemannian submersion (see {\citet[Chapter 9]{besse_einstein_1987}}). The `$R$' in $\pi_R$ stands for `right', and we will write $\pi_L:G/H \to H\backslash G/H$ for this other projection map, and $\pi = \pi_L\circ\pi_R: G \to H\backslash G/H$.
\begin{remark} \normalfont
While $G/H$ inherits the structure of a manifold from the one on $G$, $H\backslash G/H$ does not. See Example \ref{doub_coset_sphere_ex} and Example \ref{doub_cose_spd_ex}.
\end{remark}
\begin{theorem}\label{bernstein_cor}
Suppose that the connected unimodular Lie group $G$ is equipped with a left-invariant right-$H$-invariant Riemannian metric, where $H$ is a compact Lie subgroup of $G$. Equip $G/H$ with the Riemannian metric induced by the one on $G$. Suppose $\tilde f\in C^\infty_{{\C}} (G/H)$ is such that there is $f\in C{_\C}(H\backslash G/H)$ with $\tilde f = f\circ\pi_L$, i.e. $\tilde f$ is $H$-invariant. If $\Delta^k\tilde f$ is $L^2$-integrable with respect to the Haar measure on $G/H$ for any integer $k\geq 0$, then $f\in B_\C(H\backslash G/H)$.
\end{theorem}
\begin{proof}
$f\in C{_\C}(H\backslash G/H)$, so by Theorem \ref{lie_to_homog_thm}, to show that $f\in B_\C(H\backslash G/H)$ we only need to prove that $f\circ\pi =\tilde f\circ\pi_R\in B_\C(G)$. Noting that compact groups are unimodular, we see by {\citet[Equation 2.52]{folland_course_2015}} that $\Delta^k\tilde f$ is $L^2$-integrable with respect to the Haar measure on $G/H$ if and only $\Delta^k\tilde f\circ \pi_R$ is $L^2$-integrable with respect to the Haar measure on $G$. Now, we show in Appendix \hyperref[laplace_app]{D} that $(\Delta^k\tilde f)\circ\pi_R = \Delta^k(\tilde f\circ \pi_R)$, so Theorem \ref{bernstein_thm} tells us that $\tilde f\circ \pi_R\in B_\C(G)$.
\end{proof}
\begin{example} \normalfont
    By Corollary \ref{bernstein_cor}, any kernel $k= f(\langle\cdot ,\cdot \rangle)$ that is a smooth function of the inner product has a PD decomposition on $S^n$ (see Example \ref{doub_coset_sphere_ex} for the description of functions on the double coset space of the sphere $S^n\cong SO(n+1)/SO(n)$, i.e. of invariant kernels on $S^n\cong SO(n+1)/SO(n)$). For instance, the hyperbolic tangent kernel $k = \tanh(a\langle \cdot ,\cdot \rangle +b)$ on a sphere $S^n$, for $a,b\in\R$, has a PD decomposition, despite the fact that it is not PD in general (see {\citet[Example 5]{smola_regularization_2000}}).
\end{example}
We now consider the case where $M = G/H$ is a symmetric space of the non-compact type. We may choose $G$ and $H$ such that $G$ is connected, semisimple, has a finite center, and is non-compact, and $H$ is compact. The Riemannian metric on such a non-compact symmetric space is $G$-invariant and always arises from a left-invariant right-$H$-invariant metric on $G$ {\citep{helgason1962differential}}. In Corollary \ref{symmetric_cor} and its proof below, we will use a shorthand and say a function ``decays like" or ``grows like" another function if it decays at least as fast or grows at most as fast as another function.
\begin{corollary}\label{symmetric_cor} Suppose $G/H$ is a non-compact semisimple symmetric space, with $G$ and $H$ chosen as above. Suppose $f\in C{_\C}(H\backslash G/H)$ is a function of the Riemannian distance squared, i.e. $$f \circ \pi_L = \tilde f = g(d(eH,\cdot )^2).$$ If $g\in C^\infty_{{\C}}([0,\infty))$ is such that it and its derivatives decay exponentially at infinity, i.e.
$$g^{(k)}(x) = o(e^{-c_kx}) \text{ as } x\to \infty$$
for all $k\geq 0$ and some $c_k>0$ possibly depending on $k$, then $f\in B_\C(H\backslash G/H)$. In particular, the Gaussian function $\exp(-\lambda d(eH,\cdot )^2)\in B_\R(H\backslash G/H)$, and the corresponding Gaussian kernel $k = \exp(-\lambda d(\cdot ,\cdot )^2)$ has a PD decomposition for any $\lambda >0$.
\end{corollary}
\begin{proof}
This follows from Theorem \ref{bernstein_cor} combined with two of Helgason's formulas for non-compact semisimple symmetric spaces, the first one for integration and the other for the radial part of the Laplace-Beltrami operator. Observe that since $G$ is semisimple, it is also unimodular \cite[Chapter IV Proposition 1.4]{helgason1962differential}, so we are in a position to apply Theorem \ref{bernstein_cor}. Also note that since the metric on $G/H$ is left-invariant, any function of the distance squared $g(d(eH,\cdot ))^2$ defines a function in $C{_\C}(H\backslash G/H)$. \par
Using the notation from {\citet[Equation 13]{Said2018}}, for $\tilde f$ an integrable function on $G/H$,
$$\int_{G/H} \tilde f(x)d\nu(x) = C\int_H\int_\a \tilde f(x(a,h))D(a) da\,d\mu(h)$$
for some constant $C$ independent of $\tilde f$. Here $\a$ is an Abelian subalgebra of $\g$ resulting from the Iwasawa decomposition $\g = \h+\a+\n$ where $\g$, $\h$ are the Lie algebras of $G$, $H$ respectively, $\n$ is a nilpotent subalgebra of $\g$, $da$ is the Lebesgue measure on $\a$, $d\nu$ is a Haar measure on $G/H$, $d\mu$ is the normalized Haar measure of $H$, $x(a,h) := \exp(\text{Ad}(h)a)H$ and
$$D(a) := \prod_{\lambda>0}\sinh^{m_\lambda}(|\lambda(a)|)$$
where the product is over positive roots $\lambda:\a\to\R$ and $m_\lambda$ is the dimension of the root space corresponding to $\lambda$. See {\citet[Chapter IV Proposition 1.17]{helgason1962differential}}. Few details will be important for us, only the form of the integral and the growth rate of $D$ at infinity. \par
Let $\tilde f$ and $g$ be as in the statement of the corollary. We first show that $\tilde f$ is $L^2$-integrable. We have
$$\int_{G/H} |\tilde f(x)|^2 d\nu(x) = A\int_\mathfrak{a}|\tilde f(x(a,e))|^2D(a)da = A\int_\mathfrak{a}|g(\|a\|^2)|^2D(a)da$$
for some constant $A$, where we have noted that the integrand does not depend on $h$ and have integrated out that variable. The integrand is continuous, so to show that it is finite we only need to consider its behavior at infinity. Since $g$ decays exponentially and $D$ grows exponentially (by using that the roots are linear maps), the integrand decays like a Gaussian at infinity. We deduce that $\tilde f$ is $L^2$-integrable. \par
Now since $\tilde f$ is a function of distance only, $\Delta \tilde f = \Delta_r \tilde f$ where
\begin{equation}\label{rad_laplace_eq}
\Delta_r  = \sum_{i=1}^r \frac{\partial^2}{\partial a_i^2} +\sum_{\lambda>0}m_\lambda \coth (\lambda(a))\frac{\partial}{\partial a_\lambda}
\end{equation}
is the radial part of the Laplace-Beltrami operator, where the $a_i$ form an orthonormal basis of $\a$, the second sum is over positive roots $\lambda$, and $\lambda= B(a_\lambda,\cdot )$ where $B:\g\times\g \to \R$ is the Killing form.
See {\citet[Chapter II Theorem 5.24]{helgason_groups_2000}}. \par
Since $g$ and $d(eH,\cdot )^2$ are smooth, so is $\tilde f$. Thus, to show that for $k\geq0$, $\Delta^k \tilde f$ is $L^2$-integrable, it suffices to consider its behavior at infinity. It is not too hard to see from (\ref{rad_laplace_eq}) that since $\tilde f(x(a,h)) = g(\|a\|^2) = g(a_1^2+\dots +a_r^2)$ and the derivatives of $g$ and of $\coth$ decay exponentially at infinity, $\Delta^k \tilde f$ decays like a Gaussian at infinity. Hence, arguing as above we deduce that $\Delta^k \tilde f$ is $L^2$-integrable. So we are done by Theorem \ref{bernstein_cor}.
\end{proof}

\begin{remark} \normalfont
    We can show the PD decomposability of the Gaussian kernel for reductive non-compact symmetric spaces too. For instance, the space of $n\times n$ symmetric positive definite matrices $\mathbb S_{++}^n\cong GL(n)/O(n)$ is not a semisimple symmetric space. However, we can also write it as $\mathbb S_{++}^n \cong SL(n)/SO(n) \times \R_{>0}$, with the diffeomorphism
    \begin{equation}\label{spd_det_eq}
    \begin{aligned}
        \mathbb S_{++}^n &\to SL(n)/SO(n) \times \R_{>0} \\
        X &\mapsto (X/\det(X)^n, \det(X))
    \end{aligned} 
    \end{equation}
    where $SL(n)/SO(n)$ is viewed as the space of symmetric positive definite matrices with unit determinant. Then, if we write $d$ for the Riemannian distance on $\mathbb S_{++}^n\cong GL(n)/O(n)$, sometimes called the affine-invariant Riemannian distance, and $d'$ for the Riemannian distance on $SL(n)/SO(n)$, we can view the map (\ref{spd_det_eq}) as an isometry:
    $$d(X,Y) = d'(X/\det(X)^n, Y/\det(Y)^n)+|\log(\det(X))-\log(\det(Y))|$$
    for all $X,Y\in \mathbb S_{++}^n$. Therefore, the Gaussian kernel on $\mathbb S_{++}^n$ can be viewed as
    $$
    \begin{aligned}
    &\exp(-\lambda d(X,Y)^2) \\
    &= \exp(-\lambda d'(X/\det(X)^n, Y/\det(Y)^n)^2)\exp(-\lambda |\log(\det(X))-\log(\det(Y))|^2)
    \end{aligned}
    $$
    a product of a PD decomposable kernel by Corollary \ref{symmetric_cor} with a PD kernel, which is PD decomposable by Example \ref{product_ex}.
\end{remark}
\section{Conclusion}

We have seen that kernel methods can be applied with non-PD kernels that admit a PD decomposition and that knowledge of the specific form of the decomposition is not necessary for RKKS learning. We have argued that RKKS-based methods are particularly interesting in non-Euclidean settings where PD kernels are challenging to construct. We have then related PD decomposable invariant kernels on locally compact quotient spaces $G/H$ to Hermitian PD decomposable functions on the double coset space $H\backslash G/H$. This allowed us to leverage the extensive harmonic analytic literature on the Fourier-Stieltjes algebra $B_\C(G)$ to describe the invariant kernels that admit a PD decomposition. In particular, assuming some geometry on $G/H$, we showed that smoothness and appropriate decay of the derivatives at infinity is sufficient for the PD decomposability of a kernel, providing weak and verifiable sufficient conditions for a kernel to admit a PD decomposition. This work provides a theoretical foundation for applications of kernel methods on non-Euclidean data spaces. \par
% We note that there remain open questions on the characterization of universal kernels on such data spaces. We aim to investigate such questions in future work.

% Acknowledgements and Disclosure of Funding should go at the end, before appendices and references

\acks{The authors acknowledge financial support from the School of Physical and Mathematical Sciences, the Talent Recruitment and Career Support (TRACS) Office, and the Presidential Postdoctoral Programme at Nanyang Technological University (NTU). \\
%{We thank the anonymous referees for their careful reading and comments that helped improve the manuscript.}
}


\newpage

\appendix
{
\section*{Appendix A.}\label{representer_app}
In this appendix we expand on the RKKS representer theorem described in Section \ref{pdd_ker_sec}.}

{
Recall that we are given
\begin{enumerate}
    \item $\mathcal D = \{(x_1,y_1),\dots, (x_N,y_N)\}\subset (X\times\R)^N$ a finite data set,
    \item $g:\R\to \R$ a strictly monotonic differentiable function,
    \item $L:\K \times (X\times\R)^N \to \R$ a loss functional, determined exclusively through function evaluations, Fréchet differentiable in the first argument,
    \item $I:\K \times (X\times\R)^N \to \R^m$ and $E:\K \times (X\times\R)^N \to \R^l$ functionals determined exclusively through function evaluations and Fréchet differentiable in the first argument,
\end{enumerate}
For any hope at formulating, let alone proving, an RKKS representer theorem, a key additional requirement is that the set defined by the constraints
    $$
        \mathcal M := \{f\in\K : I(f,\mathcal D) \leq 0 \text{ and } E(f,\mathcal D) = 0\}
    $$
is an infinite-dimensional submanifold (with boundary) of $\K$. By infinite-dimensional submanifold we mean here a Hilbert manifold with Fréchet differentiable transition maps, since $\K$ is given the topology of the Hilbert space $\mathcal H_+ \times \mathcal H_-$ (see Definition \ref{krein_def}). Given this, we have a well-defined notion of a solution to the problem
\begin{equation*}
\begin{aligned}
    \underset{f\in\K}{\text{stabilize}}&\; L(f,\mathcal D) +g(\langle f,f\rangle_\K) \\
    \text{s.t.}&\; f\in \mathcal M.
\end{aligned}
\end{equation*}
Namely a solution $f^*$ is one for which the Fréchet derivative of the regularized loss $D[L(\cdot,\mathcal D) +g(\langle \cdot,\cdot\rangle_\K)]|_{\mathcal M}(f_*)$ vanishes on any tangent vector to the manifold $\mathcal M$.}

{
Moreover, we assume that the map $(I\times E)(\cdot , \mathcal D): {\mathcal K} \to \mathbb{R} ^{m+l}$ is a submersion. Note that when there are no inequality constraints $I$, this automatically implies that $\mathcal M$ is a manifold \cite[Theorem 3.5.4]{mta}.}

{
Given these assumptions, we can apply \cite[Corollary 3.5.29]{mta} to conclude that for each critical point $f ^\ast  \in {\mathcal K} $ of the restriction of $L(\cdot ,\mathcal D) +g(\langle \cdot ,\cdot \rangle_\K)$ to $\mathcal M$, there exist constants $\lambda\in \R^m$, $\mu\in\R^l$ such that $f ^\ast$ is a critical point of the corresponding Lagrange function defined by
    $$\mathcal L(f) = L(f,\mathcal D)+g(\langle f,f\rangle_\K)+\lambda^{\top}I(f,\mathcal D)+\mu^{\top}E(f,\mathcal D),$$
that is the Fr\'echet derivative $D\mathcal L(f^\ast ) = 0$. We recall that $\lambda_i I_i(f^\ast ,\mathcal D) = 0$ for all $i \in \left\{1, \ldots, m\right\}$, this implies that $\lambda_i$ is non-zero only when the corresponding inequality constraint is active, that is $I_i(f^\ast ,\mathcal D)<0$.}

{
By the Fréchet differentiability of the various functions, we can apply the Fréchet derivative chain rule:
    $$
    0 = D\mathcal L(f^*) = \sum_{i=1}^N\partial_{f(x_i)}(L(f^*,X)+\lambda^TI(f^*,\mathcal D)+\mu^TE(f^*,\mathcal D))k(x_i,\cdot)+ 2f^*\partial_{\langle f,f\rangle}g(\langle f^*,f^*\rangle_\K)
    $$
    where on the RHS, the derivatives are classical derivatives. So we get
    $$f^* = -\frac{1}{2\partial_{\langle f,f\rangle}g(\langle f^*,f^*\rangle_\K)}\sum_{i=1}^N\partial_{f(x_i)}(L(f^*,\mathcal D)+\lambda^TI(f^*,\mathcal D)+\mu^TE(f^*,\mathcal D))k(x_i,\cdot)$$
    where we used $g$ strictly monotonic to deduce $\partial_{\langle f,f\rangle}g(\langle f^*,f^*\rangle_\K) \neq 0$. In particular, when $g(x) = cx$ for all $x\in \R$ we get
    $$f^* = -\frac{1}{c}\sum_{i=1}^N\partial_{f(x_i)}(L(f^*,\mathcal D)+\lambda^TI(f^*,\mathcal D)+\mu^TE(f^*,\mathcal D))k(x_i,\cdot).$$}
\section*{Appendix B.}\label{averaging_app}
In this appendix, we show that when $G$ is compact, the correspondence from Proposition \ref{correspondence_prop} between Hermitian functions on $H\backslash G/H$ with continuous PD decompositions and kernels on $G/H$ with continuous PD decompositions is a surjection. As explained in Remark \ref{stat_rmk}, this requires PD decomposable invariant kernels to admit invariant PD decompositions.
\begin{theorem}
If the locally compact group $G$ is compact and $H$ is a closed subgroup of $G$, then an invariant kernel $k$ on $G/H$ has a PD decomposition if and only if it has a PD decomposition into invariant kernels.
\end{theorem}
\begin{proof}
For $g\in G$, define $T_g$ to be the operator acting on maps $k:G/H\times G/H\to \C$ by
$$T_gk(aH,bH) = k(gaH,gbH)$$
for all $a,b\in G$. Let $k$ be an invariant PD decomposable kernel on $G/H$. Then, by invariance of $k$,
$$k(aH,bH)= \frac{1}{\mu(G)}\int_G T_gk(aH,bH) d\mu(g)$$
for $a,b\in G$, where $\mu$ is a Haar measure for $G$, which is finite by compactness of $G$. So writing the PD decomposition of $k$ as $k=k_+-k_-$, we get
$$k = \frac{1}{\mu(G)}\int_G T_gk_+ d\mu(g)-\frac{1}{\mu(G)}\int_G T_gk_- d\mu(g).$$
Both the first term and minus the second term are PD, and they are invariant by left-invariance of the Haar measure, so this is a PD decomposition into invariant kernels.
\end{proof}
\section*{Appendix C.}\label{tech_sec}

{The aim of this appendix is to provide an alternative general proof of Theorem \ref{lie_to_homog_thm} using the language of representation theory.}
\begin{definition}\label{representation_def} A representation of $G$ is a group homomorphism
$$\pi: G\to GL_\C(V)$$
for some complex vector space $V$. A representation is called unitary if $V$ is a Hilbert space and the range of $\pi$ consists of unitary operators, and it is called continuous if $\pi$ is continuous.
\end{definition}
\begin{proposition}\label{pd_representation_prop}
$f\in B_+(G)$ if and only if there is a continuous unitary representation $\pi: G\to GL_\C(V)$ and a vector $v\in V$ such that
\begin{equation}\label{pd_representation_eq}f = \langle \pi(\cdot )v,v\rangle.\end{equation}
\end{proposition}
\begin{proof}
$\Longleftarrow\,:$ For $f$ as in (\ref{pd_representation_eq}), $g_1,\dots,g_N \in G$ and $c_1,\dots,c_N\in \C$,
$$
\begin{aligned}
\sum_{i=1}^N\sum_{j=1}^N\overline{c_i}c_jf(g_i^{-1}g_j) &= \sum_{i=1}^N\sum_{j=1}^N\overline{c_i}c_j\langle \pi(g_i^{-1}g_j)v,v\rangle  \\
&= \sum_{i=1}^N\sum_{j=1}^N \langle c_j\pi(g_j)v, c_i\pi(g_i)v\rangle \\
&= \left\langle \sum_{j=1}^N c_j\pi(g_j)v, \sum_{i=1}^Nc_i\pi(g_i)v\right\rangle\geq 0
\end{aligned}
$$
using the fact that $\pi$ is a unitary representation. Moreover, $f$ is continuous and bounded since $\pi$ is, so $f\in B_+(G)$. \par
$\Longrightarrow\,:$ It can be obtained by combining Theorem 3.20 and  Proposition 3.35 in \cite{folland_course_2015}.
\end{proof}
\begin{proposition}
$f\in B_\C(G)$ if and only if there is a continuous unitary representation $\pi: G\to GL_\C(V)$ and vectors $v,w\in V$ such that
\begin{equation}\label{pdd_representation_eq}f = \langle \pi(\cdot )v,w\rangle.\end{equation}
\end{proposition}
\begin{proof}
$\Longrightarrow\,:$ We show that functions $f$ of the form (\ref{pdd_representation_eq}) form a complex linear space. Closure under scalar multiplication can be obtained by scaling $v$. For closure under addition, observe that
$$\langle \pi_1(\cdot )v_1,w_1\rangle+\langle \pi_2(\cdot )v_2,w_2\rangle = \langle (\pi_1\oplus\pi_2)(\cdot )(v_1\oplus v_2),w_1\oplus w_2\rangle$$
and $\pi_1\oplus\pi_2$ is a continuous unitary representation if $\pi_1$ and $\pi_2$ are. Thus, functions of the form (\ref{pdd_representation_eq}) form a complex linear space, and contain $B_+(G)$ by Proposition \ref{pd_representation_prop}, so they contain $B_\C(G)$. \par
$\Longleftarrow\,:$ See {\citet[Lemma 2.1.4]{kaniuth_fourier_2018}}.
\end{proof}
We have described PD and PD decomposable functions on $G$ in the language of representation theory. What about their analogues on $H\backslash G/H$?

\begin{lemma}\label{pd_invariant_lem}
For $f\in B_+(G)$, 
$$f = \langle \pi(\cdot )v,v\rangle$$
as in (\ref{pd_representation_eq}), we have that $f\in C{_\C}(H\backslash G/H)$ if and only if $\pi(h)v= v$ for all $h\in H$.
\end{lemma}
\begin{proof}
$\Longleftarrow\,:$ If $\pi(h)v = v$ for all $h\in H$, then 
$$f(hgh') = \langle\pi(hgh')v,v\rangle = \langle\pi(g)\pi(h')v, \pi(h^{-1})v\rangle = \langle\pi(g)v,v\rangle = f(g)$$
for all $g\in G$ and $h,h'\in H$. Thus, $f\in C{_\C}(H\backslash G/H)$. \par
$\Longrightarrow\,:$ If $f\in C{_\C}(H\backslash G/H)$, for $h\in H$,
$$\langle \pi(h)v,\pi(h)v\rangle=\langle v,v\rangle = f(e) = f(h) = \langle\pi(h)v,v\rangle$$
where the left most equality uses the fact that $\pi(h)$ is unitary. Thus,
$$\langle\pi(h)v,v\rangle = \sqrt{\langle \pi(h)v,\pi(h)v\rangle\langle v,v\rangle}$$
and by the equality condition of Cauchy-Schwarz's inequality, we deduce that $\pi(h)v = v$.
\end{proof}
For $\pi$ a linear representation of $G$ into $GL_\C(V)$, write
\begin{equation}\label{invariant_subspace_eq}W := \{w\in V : \pi(h) w= w \;\forall\; h\in H\}\end{equation}
and let $p:V\to W$ denote the orthogonal projection. In Lemma \ref{pd_invariant_lem}, we have shown $p(v) = v$ for the corresponding $\pi$ and $v$. For the more general setting $f\in B_\C(G)$, we do not necessarily have $p(v) = v$ or $p(w)=w$, but rather the following weaker lemma. \par
The proof for the forward direction of this lemma is technical, requires quite a bit of work to prove, and relies on Zorn's lemma. The reason for this difficulty is the generality of the setting we are working in, and specifically the generality of the locally compact group $G$. With such generality, there are few results that we can use about its continuous unitary representations. However, this will pay off with the generality that we shall obtain in Theorem \ref{lie_to_homog_thm}.
\begin{lemma}\label{pdd_invariant_lem}
For $f\in B_\C(G)$, 
$$f = \langle \pi(\cdot )v,w\rangle$$
as in (\ref{pdd_representation_eq}), we have that $f\in C{_\C}(H\backslash G/H)$ if and only if we can also write $f$ as
$$f = \langle \pi(\cdot )p(v),p(w)\rangle$$
with $p$ as defined above.
\end{lemma}
\begin{proof}
$\Longleftarrow\,:$ If $g\in G$ and $h,h'\in H$,
$$f(hgh') = \langle \pi(hgh')p(v),p(w)\rangle = \langle \pi(g)\pi(h')p(v),\pi(h^{-1})p(w)\rangle = \langle \pi(g)p(v),p(w)\rangle=f(g)$$
so $f\in C_{{\C}}(H\backslash G/H)$. \par
$\Longrightarrow\,:$ fix $g\in G$ and consider the set
$$
\begin{aligned}
\P := \{&Z: W\leq Z\leq V \text{ s.t. } Z=\overline{Z} \text{ and } \langle \pi(g)v,w\rangle=\langle \pi(hgh')v,q(w)\rangle \\
&\forall \; h,h'\in H\text{ where }q: V\to Z \text{ is the orthogonal projection}\}.
\end{aligned}
$$
$\P$ is partially ordered by inclusion. We want to apply Zorn's lemma to $\P$. So we break down our proof into two steps: we first show in step I that every chain---i.e. totally ordered subset---of $\P$ has a lower bound in $\P$. Zorn's lemma then tells us that $\P$ has a minimal element. Then we show in step II that this minimal element is $W$. It may seem obvious, but it requires some care since a chain may be uncountable and need not be well-ordered. II, on the other hand, will require us to look more closely at the representation $\pi$. We will expand it in irreducible components along a cyclic subgroup of $H$.\par

$I:$ Let $\Ch\subset \P$ be a chain. We want to show that $\Ch$ has a lower bound in $\P$. This lower bound will, of course, be $Y:=\bigcap_{Z\in \Ch}Z$, so we only need to show $Y\in\P$. Note that for $Z\in\Ch$, $Y\subset Z$ so $Y^\perp \supset Z^\perp$, thus taking finite sums of elements and closure over such $Z$ we get $Y^\perp \supset \overline{\sum_{Z\in\Ch}Z^\perp}$. We can also see that for $Z'\in\Ch$, $Z' \supset \left(\sum_{Z\in\Ch}Z^\perp\right)^\perp$, so taking the intersection over such $Z'$ we get $Y \supset \left(\sum_{Z\in\Ch}Z^\perp\right)^\perp$ and taking orthogonal complements $Y^\perp \subset \overline{\sum_{Z\in\Ch}Z^\perp}$. So we have shown
\begin{equation}\label{ortho_complement_eq}
Y^\perp= \overline{\sum_{Z\in\Ch}Z^\perp}.
\end{equation}
Let $r:V\to Y$, $r^{\perp}:V\to Y^\perp$ be the orthogonal projections, and for $Z\in\Ch$, let $q_Z: V\to Z$, $q_Z^\perp: V\to Z^\perp$ be the orthogonal projections. We have for such $Z$ and $h,h'\in H$,
$$\langle \pi(hgh')v,q_Z(w)\rangle = \langle \pi(g)v,w\rangle=\langle \pi(hgh')v,w\rangle = \langle \pi(hgh')v,q_Z(w)\rangle + \langle \pi(hgh')v,q_Z^\perp(w)\rangle$$
so $\langle \pi(hgh')v,q_Z^\perp(w)\rangle=0$. Similarly, to show $Y\in\P$, we only need to show $\langle \pi(hgh')v,r^\perp(w)\rangle=0$. Now by (\ref{ortho_complement_eq}), $r^\perp(w)$ can be written as a series of countably many elements of the $Z^\perp$ for $Z\in\Ch$, so $q_Z^\perp(r^\perp(w))$ gets arbitrarily close to $r^\perp(w)$ as $Z$ decreases (we don't use the notion of limit here as $\Ch$ may be uncountable). Now $q_Z^\perp(w) = q_Z^\perp(r^\perp(w))$, so $\langle \pi(hgh')v,q_Z^\perp(w)\rangle=0$ for all $Z\in\Ch$ implies $\langle \pi(hgh')v,r^\perp(w)\rangle=0$, and we are done. \par

$II:$ By Zorn's lemma, $\P$ has a minimal element, which we denote by $Z$ and write
$q:V\to Z$ for the corresponding orthogonal projection. If $q(w)=p(w)$, then by minimality of $Z$ we have $Z=W$. Otherwise, suppose first $q(w)\neq p(w)$. Then there is $h\in H$ such that $\pi(h)q(w) \neq q(w)$. Let $C_h$ be the cyclic group generated by $h$. For conciseness, we will treat the cases $C_h$ finite and $C_h$ infinite together and only use the fact that $C_h$ is Abelian. Indeed, {\citet[Theorem 4.45]{folland_course_2015}} tells us that we can decompose $\pi|_{C_h}$ into irreducible components in the form
$$\pi(h^k) = \int_{\widehat C_h} \xi(h^k)dP(\xi)$$
for all $k\in\Z$, for some regular $V$-projection-valued measure $P$ on $\widehat C_h$ (see {\citet[Theorem 1.38]{folland_course_2015}} and the following discussion for the definition of a projection-valued measure). Then
    $$
    \begin{aligned}
    \langle \pi(g)v,w\rangle &=\langle \pi(h'gh'')v,\pi(h^k)q(w)\rangle \\
    &= \left\langle \pi(h'gh'')v,\int_{\widehat C_h}\xi(h^k)dP(\xi)q(w)\right\rangle \\
    &= \int_{\widehat C_h}\xi(h^k)dP_{\pi(h'gh'')v,q(w)}(\xi) \\
    &= \F^{-1}P_{\pi(h'gh'')v,q(w)}(h^k)
    \end{aligned}
    $$
for all $k\in\Z$, where $P_{\pi(h'gh'')v,q(w)}$ is a finite complex Radon measure, i.e. $P_{\pi(h'gh'')v,q(w)}\in M_\C(\widehat C_h)$, and $\F^{-1}$ is the inverse Fourier transform. Then by injectivity of $\F^{-1}$ we must have
$$P_{\pi(h'gh'')v,q(w)} = \langle \pi(g)v,w\rangle\delta_1$$
where $\delta_1$ is the Dirac measure at $\xi = 1$. \par
    Now consider
    $$Y := \{z\in Z : \pi(h)z = z\} $$
    where $h$ is as before. So for all $h',h''\in H$,
    $$
    \begin{aligned}
    \langle \pi(g)v,w\rangle &= P_{\pi(h'gh'')v,q(w)}(1) \\
    &= \int_{\{1\}}dP_{\pi(h'gh'')v,q(w)}(\xi) \\
    &= \left\langle \pi(h'gh'')v,\int_{\{1\}} dP(\xi)q(w)\right\rangle \\
    &= \langle \pi(h'gh'')v,P(1)q(w)\rangle \\
    &= \langle \pi(h'gh'')v,r(w)\rangle
    \end{aligned}
    $$
    where $r:V\to Y$ is the orthogonal projection. Hence, $Y\in \P$. But $\pi(h)q(w) \neq q(w)$ so $r(w) \neq q(w)$, which implies $Y<Z$. This contradicts the minimality of $Z$. Thus, $Z=W$, and hence $W\in \P$. \par
    Now by noting that
    $$\langle \pi(hgh')v,p(w)\rangle = \langle q(v),\pi((hgh')^{-1})p(w)\rangle = \overline{\langle\pi(h'^{-1}g^{-1}h^{-1})p(w),v\rangle},$$
    for all $h,h'\in H$, we can apply the same argument on $v$ to obtain
    $$\langle \pi(g)v,w\rangle = \langle \pi(h'gh'')p(v),p(w)\rangle$$
    for all $h',h''\in H$. Finally, $g\in G$ was arbitrary, so we are done..
\end{proof}
We are now in a position to prove Theorem \ref{lie_to_homog_thm} in its full generality, which we restate here for convenience.
\lietohomogthm*
\begin{proof}
$\supseteq\,:$ For part 1, if $f\in B_\R(H\backslash G/H)$ then $f\in C{_\C}(H\backslash G/H)$ by definition, and $f$ has a PD decomposition, so $f\in B_\R(G)$. Part 2 is similar. \par
$\subseteq\,:$
    For part 1, let $f\in B_\R(G)\cap C{_\C}(H\backslash G/H)$. By Proposition \ref{pd_representation_prop}, there are continuous unitary representations $\pi_\pm: G \to GL_\C(V_\pm)$, and $v_\pm\in V_\pm$ such that
    $$f = \langle \pi_+(\cdot )v_+,v_+\rangle -\langle \pi_-(\cdot )v_-,v_-\rangle = \langle (\pi_+\oplus\pi_-)(\cdot )(v_+\oplus v_-), v_+\oplus(-v_-)\rangle.$$
    Define $W_\pm \leq V_\pm$ as
    $$W_\pm := \{w\in V_\pm : \pi_\pm(h) w= w \;\forall\; h\in H\}$$
    and $p_\pm: V_\pm \to W_\pm$ the orthogonal projections. For $W$ as in (\ref{invariant_subspace_eq}) and $p:V\to W$ the orthogonal projection, observe that $W=W_+\oplus W_-$ and $p = p_+\oplus p_-$. Then by Lemma \ref{pdd_invariant_lem} we have
    $$\begin{aligned}
    f &= \langle (\pi_+\oplus\pi_-)(\cdot )p(v_+\oplus v_-), p(v_+\oplus(-v_-))\rangle \\
    &= \langle (\pi_+\oplus\pi_-)(\cdot )(p_+(v_+)\oplus p_-(v_-)), p_+(v_+)\oplus p_-(-v_-)\rangle \\
    &= \langle \pi_+(\cdot )p_+(v_+),p_+(v_+)\rangle - \langle\pi_-(\cdot )p_-(v_-),p_-(v_-)\rangle.
    \end{aligned}$$
    Now $\langle\pi_\pm(\cdot )p_\pm(v_\pm),p_\pm(v_\pm)\rangle\in B_+(H\backslash G/H)$ by Lemma \ref{pd_invariant_lem}, so $f\in B_\R(H\backslash G/H)$ and we are done. \par
    For part 2, the proof follows in a similar way by first writing $f$ as 
    $$f = \langle \pi_+(\cdot )v_+,v_+\rangle -\langle \pi_-(\cdot )v_-,v_-\rangle +i\langle \pi_{+i}(\cdot )v_{+i},v_{+i}\rangle-i\langle \pi_{-i}(\cdot )v_{-i},v_{-i}\rangle.$$
\end{proof}
\begin{remark} \normalfont
    Note that the proof of Lemma \ref{pdd_invariant_lem} also gives us the results $B_\R(G) \cap C{_\C}(G/H) = B_\R(G/H)$ and $B_\C(G) \cap C{_\C}(G/H) = B_\C(G/H)$.
\end{remark}

\section*{Appendix D.}\label{laplace_app}
In this appendix we investigate properties of the Laplace-Beltrami operator on the Lie group $G$ equipped with a left-invariant Riemannian metric, which are needed in Section \ref{suf_cond_sec}. We will use the following definition of the Laplace-Beltrami operator:
$$\Delta := \sum_{i=1}^n(X_i^2-\nabla_{X_i}X_i)$$
locally around a point $g\in G$, where $X_1,\dots,X_n$ is a local orthonormal frame of vector fields around $g$ and $\nabla$ is the Levi-Civita connection with respect to the metric on $G$. This definition of $\Delta$ is independent of the choice of local orthonormal frame. \par
In the proof of Theorem \ref{bernstein_thm}, we implicitly use the following:
\begin{proposition}
Let $G$ be a connected unimodular Lie group. If $X_1,\dots,X_n$ is a basis of left-invariant vector fields for $G$, and we equip $G$ with the left-invariant Riemannian metric making $X_1,\dots,X_n$ orthonormal, then the Laplace-Beltrami operator $\Delta$ can be written as
    $$\Delta = \sum_{i=1}^nX_i^2.$$
\end{proposition}
\begin{proof}
We have
$$\Delta = \sum_{i=1}^n(X_i^2-\nabla_{X_i}X_i),$$
so we need to show that $\sum_{i=1}^n\nabla_{X_i}X_i=0$. \par
$G$ being unimodular is equivalent to the adjoint representation $\text{ad}: \g\to \text{Lie}(\text{Aut}(\g))$ of the Lie algebra $\g$ of $G$ having vanishing trace at every $X\in\g$ \cite[Lemma 6.3]{milnor_curvatures_1976}, where $\text{Lie}(\text{Aut}(\g))$ is the Lie algebra of $\text{Aut}(\g)$, the Lie group of linear automorphisms of $\g$. So for all $i$,
\begin{align*}
0&=\tr(\text{ad}(X_i)) \\
&= \sum_{j=1}^n\langle [X_i,X_j], X_j\rangle_\g \\
&=\sum_{j=1}^n(\underbrace{\langle\nabla_{X_i}X_j,X_j\rangle_\g}_{=0}-\langle\nabla_{X_j}X_i,X_j\rangle_\g) \\
&= -\sum_{j=1}^n\langle\nabla_{X_j}X_i,X_j\rangle_\g \\
&= \sum_{j=1}^n\langle\nabla_{X_j}X_j,X_i\rangle_\g \\
&= \left\langle\sum_{j=1}^n\nabla_{X_j}X_j,X_i\right\rangle_\g
\end{align*}
where we used the orthonormality of the $X_j$ in the fourth and fifth equality. So $\sum_{i=1}^n\nabla_{X_i}X_i=0$ at $e$, and thus everywhere on $G$ by left-invariance. 
\end{proof}
Now in the proof of Theorem \ref{bernstein_cor}, we need the following:
\begin{proposition}
Let $G$ be a Lie group and let $H$ be a compact Lie subgroup of $G$. Suppose $G$ is equipped with a left-invariant, right-$H$-invariant Riemannian metric. Equip $G/H$ with the Riemannian metric induced by the one on $G$. Then, for $\tilde f\in C^\infty_{{\C}}(G/H)$,
$$\Delta(\tilde f\circ \pi_R) = \Delta \tilde f\circ \pi_R$$
where $\pi_R:G\to G/H$ is the projection, and $\Delta$ is the Laplace-Beltrami operator with respect to the respective metrics.
\end{proposition}
\begin{proof}
    This follows many of the steps of {\citet[proof of Theorem 1.5]{bergery_laplacians_1982}}. Write $n$ and $m$ for the dimensions of $G$ and $H$, respectively. For $g\in G$, take a local orthonormal frame $\tilde X_1,\dots, \tilde X_{n-m}$ for $T_{gH}G/H$ around $gH$. This lifts to a local orthonormal frame $X_1,\dots,X_{n-m}$ for $\Ker (\pi_{R*})^\perp$ around $g$, subbundle of $TG$. So $\pi_{R*}X_i = \tilde X_i$ for all $i$. Also take a local orthonormal frame $U_1,\dots,U_m$ for $\Ker (\pi_{R*})$ around $g$. Then, the Laplace-Beltrami operator on $G$ can be locally written as
$$\Delta = \sum_{i=1}^{n-m}(X_i^2-\nabla_{X_i}X_i) + \sum_{i=1}^m(U_i^2-\nabla_{U_i}U_i)$$
and the Laplace-Beltrami operator on $G/H$ can be locally written as
$$\Delta = \sum_{i=1}^{n-m}(\tilde X_i^2-\nabla_{\tilde X_i}\tilde X_i)$$
where $\nabla$ are the Levi-Civita connections induced by the respective metrics. By the chain rule, $U(\tilde f\circ\pi_R) = \pi_{R*} U(\tilde f)\circ\pi_R = 0$ for all $U\in \Ker(\pi_{R*})$. Similarly, the chain rule gives $X_i^2(\tilde f\circ\pi_{R*}) =\pi_{R*}X_i\pi_{R*}X_i(\tilde f)\circ\pi_R =\tilde X_i^2 (\tilde f)\circ\pi_R$ for all $i$. Now by {\citet[Theorem 9.80]{besse_einstein_1987}} if $U\in \Ker(\pi_{R*})$ then $\nabla_UU\in \Ker (\pi_{R*})$ (we use here the left-invariance and right-$H$-invariance of the metric on $G$, and the compactness of $H$). Finally, by {\citet[Definition 9.23 \& following discussion]{besse_einstein_1987}}, $\nabla_{X_i}X_i (\tilde f\circ\pi_R) = \pi_{R*}(\nabla_{X_i}X_i)(\tilde f)\circ\pi_R = \nabla_{\tilde X_i}\tilde X_i (\tilde f)\circ \pi_R$ for all $i$. Thus, we obtain

\begin{align}
\Delta(\tilde f\circ \pi_R) &= \sum_{i=1}^{n-m}(X_i^2-\nabla_{X_i}X_i)(\tilde f\circ\pi_R) + \sum_{i=1}^m(U_i^2-\nabla_{U_i}U_i)(\tilde f\circ\pi_R) \\
&= \sum_{i=1}^{n-m}(\tilde X_i^2-\nabla_{\tilde X_i}\tilde X_i)(\tilde f)\circ\pi_R \\
&= \Delta \tilde f\circ\pi_R.
\end{align}

\end{proof}

\vskip 0.2in
\bibliography{references}

\end{document}